\documentclass[]{style/iromlab}

\usepackage{multicol}
\usepackage{xspace}
\usepackage{amsmath, amsfonts, amssymb, amsthm}
\usepackage{enumerate}
\usepackage[inline]{enumitem}
\usepackage{mathtools}
\usepackage{subcaption}
\usepackage{graphicx}
\usepackage{longtable,tabularx}
\usepackage{placeins} 
\usepackage{float}
\usepackage{multirow}
\usepackage{bbm}
\usepackage{wrapfig}
\usepackage{threeparttable}
\usepackage{balance}
\usepackage[ruled,algo2e]{algorithm2e}
\usepackage{algpseudocode}
\usepackage{adjustbox}
\usepackage{booktabs}
\usepackage{bm}
\usepackage{etoolbox}
\usepackage{microtype}
\usepackage[title]{appendix}
\usepackage{units}
\usepackage{cleveref}
\usepackage{tcolorbox}

\DeclareDocumentEnvironment{example}{}{\noindent\textbf{Running example:}\itshape}{}

\usepackage{ifthen}
\newboolean{include-notes}
\newboolean{include-new}
\newboolean{include-remove}
\setboolean{include-notes}{true}
\setboolean{include-new}{false}
\setboolean{include-remove}{false}

\usepackage[dvipsnames]{xcolor}
\usepackage[normalem]{ulem}
\newcommand{\justin}[1]{\ifthenelse{\boolean{include-notes}}{\textcolor{orange}{\textbf{Jaime:} #1}}{}}

\newcommand{\princeton}[1]{\ifthenelse{\boolean{include-notes}}{\textcolor{orange}{#1}}{}}

\newcommand*{\algname}{\textcolor{orange}{\textsc{AnyRank}}\xspace}
\newcommand{\Hcomp}{\mathbb{H}}
\newcommand{\Hyp}{\mathcal{H}}
\newcommand{\A}{\textbf{A}}
\newcommand{\B}{\textbf{B}}
\newcommand{\C}{\textbf{C}}
\newcommand{\D}{\textbf{D}}
\newcommand{\E}{\textbf{E}}
\newcommand{\F}{\textbf{F}}
\newcommand{\G}{\textbf{G}}
\newcommand{\partialtxt}[1]{\textcolor[HTML]{B8860B}{#1}}

\newtheorem{theorem}{Theorem}

\newtheorem{corollary}{Corollary}

\theoremstyle{definition}
\newtheorem{remark}{Remark}

\newbool{extended}
\setbool{extended}{false}

\makeatletter
\newcommand{\longdash}[1][2em]{%
  \makebox[#1]{$\m@th\smash-\mkern-7mu\cleaders\hbox{$\mkern-2mu\smash-\mkern-2mu$}\hfill\mkern-7mu\smash-$}}
\makeatother
\newcommand{\omitskip}{\kern-\arraycolsep}

\tcbset{
  promptstyle/.style={
    colback=gray!10,
    colframe=gray!60,
    boxrule=0.5pt,
    arc=2pt,
    outer arc=2pt,
    left=4pt,
    right=4pt,
    top=4pt,
    bottom=4pt,
    fonttitle=\bfseries,
    sharp corners=south,
  }
}

\tcbset{
  promptstyle_prompt_iuq/.style={
    colback=prompt_color!10,
    colframe=prompt_color!60,
    coltitle=white,
    boxrule=1.0pt,
    arc=2pt,
    outer arc=2pt,
    left=4pt,
    right=4pt,
    top=4pt,
    bottom=4pt,
    fonttitle=\bfseries,
  }
}

\tcbset{
  promptstyle_claude/.style={
    colback=claude_color!10,
    colframe=claude_color!60,
    coltitle=black,
    boxrule=1.0pt,
    arc=2pt,
    outer arc=2pt,
    left=4pt,
    right=4pt,
    top=4pt,
    bottom=4pt,
    fonttitle=\bfseries,
  }
}

\tcbset{
  promptstyle_deepseek/.style={
    colback=deepseek_color!10,
    colframe=deepseek_color!60,
    coltitle=black,
    boxrule=1.0pt,
    arc=2pt,
    outer arc=2pt,
    left=4pt,
    right=4pt,
    top=4pt,
    bottom=4pt,
    fonttitle=\bfseries,
  }
}

\tcbset{
  promptstyle_o3_mini/.style={
    colback=o3_mini_color!10,
    colframe=o3_mini_color!60,
    coltitle=black,
    boxrule=1.0pt,
    arc=2pt,
    outer arc=2pt,
    left=4pt,
    right=4pt,
    top=4pt,
    bottom=4pt,
    fonttitle=\bfseries,
  }
}

\author[1*]{Apurva Badithela}
\author[2*]{David Snyder}
\author[1]{Tenny Yin}
\author[1]{Jade Sceats}
\author[2]{George Pappas}
\author[2]{Nikolai Matni}
\author[3]{Masha Itkina}
\author[3]{Haruki Nishimura}
\author[1]{Anirudha Majumdar}

\affiliation[1]{Princeton University}
\affiliation[2]{University of Pennsylvania}
\affiliation[3]{Toyota Research Institute}
\contribution[*]{Equal contribution.}

\setlist[enumerate]{align=left,leftmargin=8mm,labelsep*=0pt}
\setlist[enumerate,1]{label=(\alph*), ref=(\alph*)}
\setlist[enumerate,2]{label=(\roman*), ref=(\roman*),leftmargin=9mm}

\begin{document}

\title{\algname:\\
{\LARGE Reliable Anytime Policy Ranking via Sequential Testing}}

\abstract{
Due to the high cost of hardware experiments, robot policies are predominantly evaluated by comparing their empirical performance on a limited number of hardware trials. However, finite-sample statistical uncertainty can lead to erroneous conclusions, especially when ranking multiple policies using a fixed batch of evaluation data. We propose \algname: a \emph{sequential} method for multi-policy ranking that (i) \emph{adaptively allocates trials} to hard-to-separate policies, (ii) provides (partial) rankings at \emph{any stopping time} while bounding the probability of an incorrect ranking, and (iii) optionally \emph{exploits prior information} to prioritize hardware evaluations more efficiently. This prior information could be informed from expert intuition, a small set of held-out hardware evaluations, or from policy rollouts in a simulator or an action-conditioned video world model. Notably, the prior can be imperfect; while a strong prior can provide efficiency gains on hardware, any inaccuracies in prior policy rankings do not invalidate the statistical assurances of \algname. We achieve the aforementioned properties through novel theoretical contributions that combine sequential testing with multiple-hypothesis graphical testing approaches. Experiments in simulation and real-world policy evaluation benchmarks show that \algname requires up to 50\% fewer hardware evaluations compared to existing statistically rigorous approaches, while guarding against the incorrect rankings produced by heuristic methods used by state-of-the-art benchmarks. 
}


\website{
https://abadithela.github.io/anyrank-page/  
}
{
abadithela.github.io/anyrank-page   
}

\code{
https://github.com/irom-princeton/multi_hypothesis_test 
}
{
github.com/anyrank  
}


\maketitle



\section{Introduction}
\label{sec:intro}
Robot learning faces an evaluation challenge. As policies become increasingly capable of generalizing to a broad range of tasks, environments, and embodiments, rigorous evaluation and benchmarking become critical bottlenecks~\cite{kress2024robot,atreya2025roboarena,li2024evaluating,anwar2025efficient,barreiros2026careful,gao2025taxonomy}. In particular, the analytical opacity of state-of-the-art (SOTA) policies places a significant burden on \emph{empirical} evaluation. 
 How can we identify the most critical design choices for our policy architecture, data mixture, or training algorithm? Can we confidently conclude if improved performance resulted from a particular intervention rather than another? How can we determine if our proposed method is genuinely better than baselines? Such questions demand \emph{reliable} policy \emph{ranking} methods using empirical measurements of performance.  

In contrast to domains such as computer vision and natural language processing, evaluation in robotics is fundamentally limited by the time- and labor-intensive nature of hardware experiments. As a result, policy performance is usually summarized with a \emph{point estimate} (e.g., empirical mean) of task-relevant performance metrics computed over a small number ($20-30$) of hardware rollouts \cite{kress2024robot}. However, comparing policies in this way is inherently unreliable due to finite-sample statistical variations in tasks, environments, embodiments, and policy stochasticity. For instance, the ranking of two closely competing policies can be easily flipped if one policy achieves a few more successful rollouts than the other by chance. This risk of incorrect inferences becomes particularly pronounced when a large number of policies must be ranked using a fixed batch of evaluation data. 
\begin{figure}[t]
    \centering
    \includegraphics[width=0.8\linewidth]{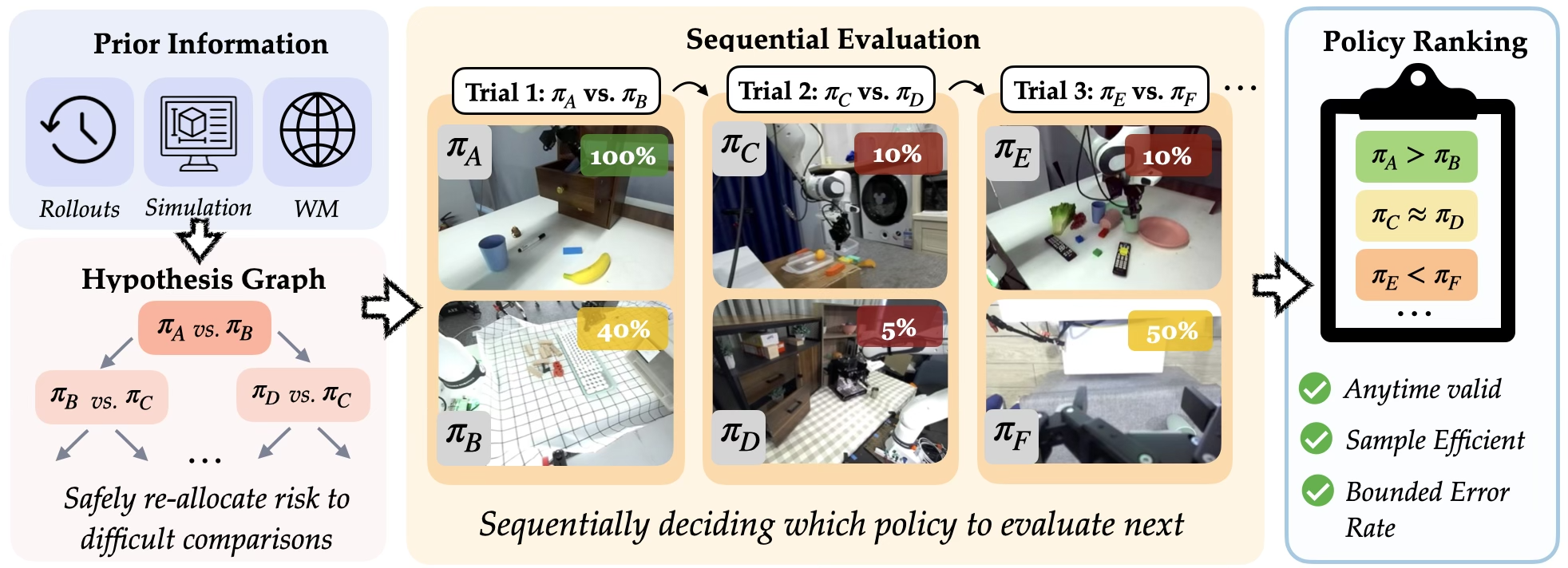}
    \caption{\textbf{Reliable Policy Ranking with \algname.} Prior information is leveraged to form a weighted multi-hypothesis graph, which in turn informs an active evaluation protocol to minimize expected sample complexity. The resulting ranking provably maintains a pre-specified error rate.}
    \label{fig:banner}
\end{figure}
Commonly used approaches~\cite{welch_generalization_1947} for producing confidence intervals or p-values often make asymptotic assumptions that are not valid for a finite number of samples. Moreover, comparisons based on confidence intervals tend to have less statistical power than direct hypothesis testing~\cite{greenland2016statistical}, leading to highly indeterminate rankings with many ties. On the other hand, preference-based ranking methods such as Bradley-Terry models~\cite{bradley-terry} tend to result in definitive rankings, but can have high error rates in the non-asymptotic regime, as we show in our experiments. While simulators and world models offer promising pathways to scalable evaluation~\cite{yin2026playworld,guo2025ctrl,li2024evaluating,team2025evaluating,ranawaka2026simfoundry,yang2026robolab}, the resulting inferences do not necessarily reflect real-world performance. As such, estimates from simulation can serve as imperfect priors that must be combined with gold-standard hardware trials for reliable policy ranking~\cite{badithela_reliable_2025}. 

We posit that an ideal framework for policy ranking in robotics should satisfy three primary desiderata. First, it should produce \emph{statistically reliable} conclusions even with the small numbers of experimental trials that are common in robotics. Specifically, we must bound the probability of an incorrect ranking to any given practitioner-specified value. The resulting ranking should be \emph{anytime-valid} by producing a (possibly partial) ranking at any decision time, while allowing for the possibility of continued evaluations without the risk of p-hacking~\cite{john2012measuring}. Second, it should be \emph{sample efficient} by adaptively and actively allocating trials to hard-to-separate policies. Finally, an ideal policy ranking method should allow for the \emph{incorporation of priors} (e.g., from researcher intuition from preliminary experiments in simulation, held-out hardware evaluations, or rollouts in a world model) to improve sample efficiency while guarding against imperfectly specified priors. 

\noindent\textbf{Statement of Contributions.} 
We propose \algname, a framework for policy ranking that satisfies the properties highlighted above: statistical reliability, anytime validity, sample efficiency, and safe incorporation of priors. In order to achieve this, we formulate policy ranking as a \emph{multiple-hypothesis testing} problem where data is collected \emph{sequentially}. Our key theoretical contribution is a novel combination of safe anytime-valid inference~\cite{ramdas_game-theoretic_2023} with graphical procedures for multiple-hypothesis testing~\cite{bretz2009graphical}. The resulting algorithm \emph{actively} allocates evaluations to policies that are more difficult to distinguish statistically, adaptively assigns the global confidence budget to different hypotheses while bounding the probability of incorrect rankings (i.e., family-wise error rate control~\cite{hartog_family-wise_2025}), and initializes hyperparameters in the testing procedure to leverage prior information. Extensive experiments on real-world policy evaluation benchmarks such as RoboArena~\cite{atreya2025roboarena} and the TRI-LBM dataset~\cite{barreiros2026careful} demonstrate that our method requires up to 50\% fewer hardware evaluations compared to existing statistically rigorous approaches, while guarding against the incorrect rankings produced by imperfect priors or heuristic methods used by state-of-the-art benchmarks. 

\section{Related Work}
\label{sec:related}

\textbf{Heuristic Approaches.}
For large-scale hardware evaluation,~\citet{atreya2025roboarena} propose a distributed evaluation approach, using Bradley-Terry models~\cite{bradley-terry} to rank policies from human preferences. This approach does not provide statistical guarantees of a correct ranking and can have high rates of error in practice (\Cref{sec:evaluation}). 
In contrast, our approach provides policy rankings that are finite-sample valid, and can use a prior, e.g., from heldout evaluations, simulation~\cite{nasiriany_robocasa_2024,yang2026robolab,li2024evaluating,kim2026molmospaces,jain2025polaris}, or a world model~\cite{team2025evaluating,guo2025ctrl,tseng2026sc3,quevedo2025worldgymworldmodelenvironment,yin2026playworld}, to inform policy evaluations on hardware.

\noindent\textbf{Sequential Hypothesis Testing.}
Our work builds upon ideas from the Neyman-Pearson statistical testing framework \cite{neyman_ix_1997}; for policy comparison, this framework has primarily been applied in fixed-sample-size contexts (i.e., \emph{batch} evaluation methods) \cite{barnard_significance_1947, welch_generalization_1947, bickel_mathematical_2015}. Allowing the evaluator to adaptively stop trials (termed anytime stopping \cite{siegmund_sequential_1985}) can invalidate statistical assurances by inflating the Type-I error (i.e., false positive) rate unless appropriate corrections are applied \cite{ramdas_game-theoretic_2023}; this phenomenon underlies many common forms of p-hacking \cite{greenland2016statistical}. Further, robotics literature typically relies on only asymptotically valid statistical procedures, even though hardware evaluation budgets are often too small for such approximations to be reliable \cite{valle_evaluating_2026, ullman_challenges_2021, rana_benchmark_2020, lingg_building_2023, correia_exploring_2018}.
In contrast, sequential \cite{siegmund_sequential_1985}, nonasymptotic approaches \cite{wasserman_all_2006, wainwright_high-dimensional_2019, samworth_modern_2026} allow the evaluator \emph{to choose when to stop evaluations} while ensuring control of Type-I error. This data adaptivity promises significant and interpretable improvement in sample efficiency via allowing feedback from the current evaluation evidence \cite{wald_sequential_1945, wald_optimum_1948, lorden_nearly-optimal_1977, lai_asymptotically_1985}. Recent work in safe anytime-valid inference (SAVI) \cite{howard_time-uniform_2021, ramdas_admissible_2022, ramdas_game-theoretic_2023, lindon_anytime-valid_2025} constructs robust families of stochastic processes that implicitly account for anytime stopping biases in general, nonparametric settings, allowing the evaluator to adapt online to save evaluation effort under general progress metrics. Building on SAVI methods, efficient, sequential and non-asymptotic pairwise robot policy comparison methods have been developed for general progress metrics~\cite{snyder2026beyond}.

\noindent\textbf{Rigorous Robot Policy Evaluation.}
Rigorous robot policy evaluation has emphasized constructing confidence intervals to estimate the mean performance of a single policy \cite{badithela_reliable_2025,luo_sim2val_2025,vincent_how_2024}, or two-policy comparison~\cite{snyder2025your,snyder2026beyond}. Ranking multiple policies, however, is a composite, multi-hypothesis testing problem which prior work only considers via inefficient Bonferroni (i.e., union bound) corrections \cite{vincent_how_2024, snyder2025your, snyder2026beyond, barreiros2026careful, lin2026systematic}. Such corrections can significantly limit statistical power because the number of pairwise hypotheses scales quadratically in the number of policies. We address this limitation using graphical hypothesis testing~\cite{bretz2009graphical,wiens_fixed_2003}, which improves power by reallocating risk budgets across hypotheses. While graphical procedures are well established for batch testing, we extend them to adaptive sampling and sequential evaluation. Together, these contributions improve efficiency over classical Bonferroni corrections while maintaining statistical validity.


\section{Problem Formulation: Policy Ranking as Hypothesis Testing}
\label{sec:preliminaries}




Given a set of $M$ robot policies, our goal is to rank the policies according to their mean performance on a (complicated and unknown) real-world task and environment distribution \(\mathcal{D}\). We assume that policy evaluations are carried out on i.i.d. samples \(e \sim \mathcal{D}\) of tasks and environments. Let \(R\) be a shared, real-valued and bounded metric (as in \cite{snyder2026beyond}) of task performance (e.g., binary success or partial progress score). We denote by \(r_j(e)\) the performance of policy \(\pi_j\) on environment~\(e\). We let $\mu_j := \mathbb{E}_{e \sim \mathcal{D}}[r_j(e)]$ denote the (unknown) true expected value of the performance metric for $\pi_j$. Thus, policy ranking refers to reliably ordering $\{\mu_j\}_{j=1}^M$ using a small number of evaluation rollouts. 

\textbf{Pairwise Policy Comparison.}
Consider a pair of robot policies \(\pi_j, \pi_{j'}\) with (unknown) true means $\mu_j, \mu_{j'}$. To rank them, we require two one-sided hypothesis tests:
\begin{equation*}
    \begin{split}
        \text{(1) }\Hyp^{0}_{j,j'}\text{ (null): }& \mathbb{E}_{e \sim \mathcal{D}}[r_j(e)] \geq \mathbb{E}_{e \sim \mathcal{D}}[r_{j'}(e)] \text{, } \Hyp^{1}_{j,j'}\text{ (alternative): } \mathbb{E}_{e \sim \mathcal{D}}[r_j(e)] < \mathbb{E}_{e \sim \mathcal{D}}[r_{j'}(e)]; \\
        \text{(2) }\Hyp^{0}_{j',j}\text{ (null): }& \mathbb{E}_{e \sim \mathcal{D}}[r_{j'}(e)] \geq \mathbb{E}_{e \sim \mathcal{D}}[r_j(e)] \text{, } \Hyp^{1}_{j',j}\text{ (alternative): }\mathbb{E}_{e \sim \mathcal{D}}[r_{j'}(e)] < \mathbb{E}_{e \sim \mathcal{D}}[r_j(e)].
    \end{split}
\end{equation*}
If test (1) rejects its null hypothesis, then we can conclude that $\mu_j < \mu_{j'}$, and vice-versa if test (2) rejects. Due to uncertainty in a limited number of rollouts, the probability of incorrect inference from hypothesis testing should be bounded with a pre-specified risk level \(\alpha \in (0,1)\). Concretely, if $\Hyp^{0}_{j,j'}$ is indeed true, then a test procedure controls the Type-1 error at rate \(\alpha\) if the probability of the test falsely concluding that $\Hyp^{1}_{j,j'}$ is true is bounded by $\alpha$. For a two-sided test at level \(\alpha\), we run both one-sided tests at budget \(\alpha\), as it is highly unlikely for two policies to have identical mean performance over complex distribution of tasks and environments (see~\Cref{rem:onesided}).

\textbf{Multiple Policy Comparison.} For multi-policy ranking, a sufficient condition is to consistently decide $K = \binom{M}{2}$ pairwise comparisons, yielding a multi-hypothesis testing problem $\Hcomp$. A ranking is \emph{complete} if all policy pairs are correctly decided, and \emph{partial} if a fraction of pairs are correctly ranked. The probability of a test incorrectly deciding at least one pair is the family-wise error rate (FWER). Standard approaches to control FWER split the risk budget between all pairwise comparisons~\cite{barreiros2026careful}, which  substantially reduces statistical power because the number of pairwise tests grows quadratically in \(M\), and hypotheses with smaller risk budgets require more hardware evaluations. 

\textbf{Problem Statement.} Given $M$ robot policies and risk level \(\alpha^*\), the goal is to rank them in as few evaluations as possible while maintaining high statistical power and FWER control. We seek a sampling and stopping test procedure \(S^*\) that balances expected sample complexity \(\mathbb{E}[T](S)\) and expected statistical power (\(1-\mathbb{E}[F](S)\)) while maintaining FWER control at \(\alpha^*\): 
\begin{equation}
    \label{multi_objective_opt}
        S^* = \arg \inf_{S \in \mathcal{S}}  \hspace{2mm} \bigg[ \mathbb{E}[T](S) + \lambda \cdot \mathbb{E}[F](S) \bigg] \quad \text{s.t.} \hspace{2mm} \text{FWER}(S) \leq \alpha^*,
\end{equation}
where \(\mathbb{E}[F](S)\) is the expected fraction of pairwise hypotheses that the test procedure \(S\) will fail to reject. While~\Cref{multi_objective_opt} is difficult to solve in general, our goal is to design a practical evaluation framework that approaches the optimal tradeoff while strictly controlling FWER. This requires an \emph{active} evaluation strategy: different pairwise hypotheses may require different amounts of evidence to reach a decision, and choosing which policy to evaluate next is important for efficient policy ranking.

\section{Method: Multi-Policy Ranking via Sequential Graphical Testing} 
\label{sec:Method}
We propose~\algname~as a statistically rigorous policy ranking procedure that decomposes the ranking problem into two fundamental components. First, the problem of multi-policy comparison is reduced to a structured set of pairwise policy comparison decisions using \emph{graphical procedures} for principled multi-hypothesis correction. Our first novel contribution is providing a general framework to incorporate prior knowledge (e.g., from simulation) in the graph structure to accelerate policy ranking. Next, evaluation of pairwise comparisons is addressed using the sequential, nonparametric evaluation procedures of \citet{snyder2026beyond}. To combine these procedures requires carefully reconciling \emph{batch} graphical procedures with \emph{sequential} evaluation methods to ensure FWER control. Proving that~\algname~indeed maintains FWER control is our second main technical contribution.



\textbf{Graphical Methods.}
Graphical testing procedures~\cite{bretz2009graphical} are SOTA for \emph{batch} multiple hypothesis testing. Each hypothesis is initially assigned a fraction of the global risk budget \(\alpha^*\), $\alpha^{[i]}_0$, which is reallocated to other hypotheses upon rejection. Second, it specifies the \emph{reallocation structure} $\{w^{[ii']} \geq 0: i,i' \in \{1, ..., K\}\}$, such that when any hypothesis $\mathcal{H}_i$ is rejected, the budget assigned to hypothesis $\mathcal{H}_i$ is recycled to increase the error budgets of the remaining hypotheses: 
\begin{equation}
\label{eq:update}
\alpha^{[i']} \leftarrow \alpha^{[i']} + \alpha^{[i]}w^{[ii']},
\qquad
w^{[i'i'']} \leftarrow \frac{w^{[i'i'']} + w^{[i'i]}w^{[ii'']}}{1 - w^{[i'i]}w^{[ii']}},
\end{equation}
where $w$: $\sum_{i'} w^{[ii']} = 1, ~\forall i \in \{1,\dots,K\}$. \Cref{fig:graph_propagation} illustrates risk reallocation in the case of a three-policy ranking problem.

\textbf{Prior Knowledge.} We incorporate prior estimates of policy performance differences into both the initial risk allocation and the graph weights for the multitest hypothesis graph (see~\Cref{sec:graph_instantiation}). Specifically, a softmax-style allocation assigns larger risk budgets to policy pairs predicted to be easier to distinguish, enabling faster decisions with fewer hardware trials and earlier risk reallocation to more difficult hypotheses. We also consider two ablations of the graphical structure, controlled by the softmax inverse temperature parameter \(\gamma\): \(\gamma=0\) yields a uniform risk allocation that ignores prior information, while \(\gamma=\infty\) yields a chain graph that uses the prior to order hypotheses from easiest to hardest to distinguish. In the batch setting, these correspond to the special cases of Holm's procedure and fixed-sequence testing, respectively (see~\Cref{sec:realloc} for details).

\begin{figure}[t]
    \centering
    \includegraphics[width=0.8\linewidth]{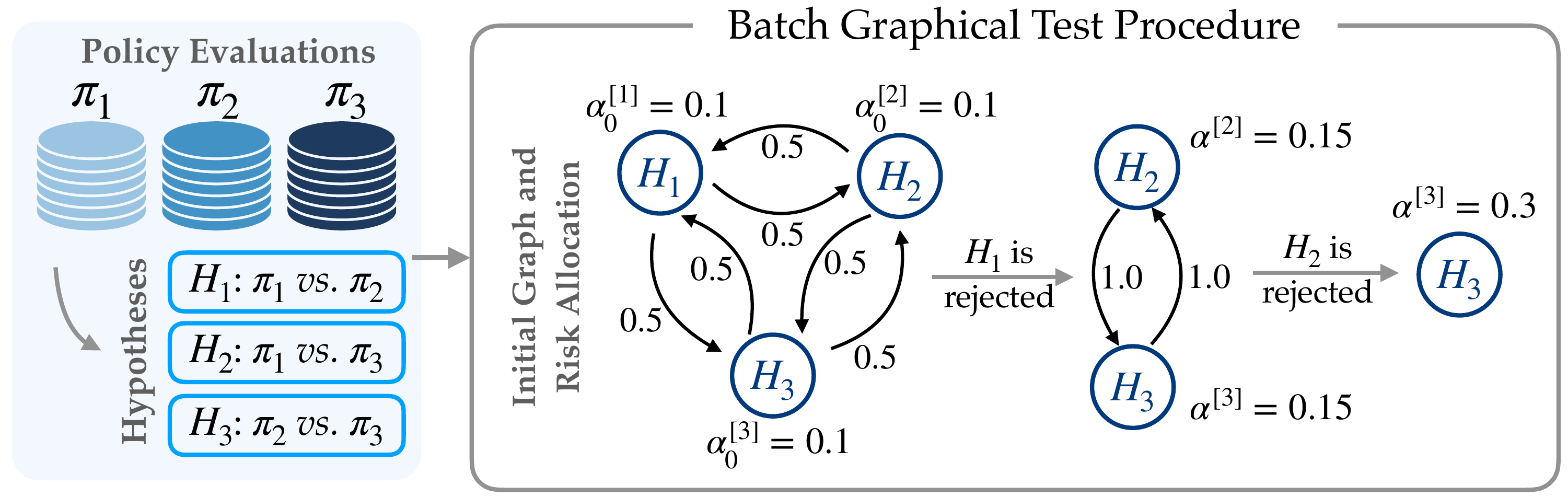}
    \caption{Graph propagation and risk reallocation as hypotheses are decided at FWER bound of $\alpha^* = 0.3$. This is a batch example, with all evaluations collected before graphical testing commences.}
    \label{fig:graph_propagation}
\end{figure}

\textbf{Robust, General, and Sequential Pairwise Comparison.}
\label{method:pairwise_comparisons}
\citet{snyder2026beyond} recently presented a SOTA \emph{sequential} method (N-SCORE) for general two-policy comparison under bounded performance measures that guarantees Type-I error control at a tunable level $\alpha$. 
N-SCORE is a safe anytime-valid inference (SAVI) procedure \cite{ramdas_game-theoretic_2023} for policy comparison; as such, it maintains a canonical p-process $(p_t)_{t \geq 0}$ (\cite{ramdas_admissible_2022}, Proposition 12.1): a sequence of p-values that gives a valid Type-I error upper bound at all possible stopping times. This p-process allows the practitioner to `peek' at the intermediate evaluation results for the atomic comparison and wait just until $p_t < \alpha_i$ is achieved before stopping the comparison -- this is the robust feedback mechanism which will enable~\algname~to utilize \emph{efficient active evaluation} online. Procedurally, for each pairwise comparison $\{\mathcal{H}_i\}_{i=1}^K$ we will instantiate an N-SCORE test which will, from any finite dataset of $t$ evaluations, induce a set of anytime p-values $\{p_{i, t}\}_{i=1}^K$ corresponding to the evidence against each atomic null hypothesis. 


\begin{algorithm2e}[t]
\caption{(Informal) Multi-Hypothesis Active Policy Ranking with \algname}
\label{alg:multitest}
\KwIn{Type-I error limit $\alpha^*$, sampling rule $S(\cdot, \cdot)$, prior means \(\{\hat{\mu}_1, \ldots, \hat{\mu}_M\}\)\;}
\KwOut{Ranking decisions $D_{t}$ and evaluation data $\mathcal{F}_t$\;}
\emph{Initialize graph} $G$, set $G_{sim}=G$\tcp*[r]{(\Cref{eq:init})}

\While(\tcp*[f]{(\Cref{eqn:advantage_by_policy})}){$\max_j \hat{p}_t(j) > 0$}{
    Actively sample policy to eval: $j_t \gets S(\mathcal{F}_t, G_{sim})$\tcp*[r]{(\Cref{eqn:hedged_advantage_following})}
    Run policy $j_t$ on $e_t \sim \mathcal{D}$: $r_{j_t}(e_t) \gets \text{Rollout}(e_t, \pi_{j_t})$\;
    Update filtration $\mathcal{F}_{t+1} = \mathcal{F}_t~\cup~\{t, r_{j_t}\}$\;
    Update anytime-valid p-values:
        $p_{t+1,i}\gets\mathrm{N\text{-}SCORE}(\mathcal{H}_i,\mathcal{F}_{t+1})$\tcp*[r]{(via \cite{snyder2026beyond})}
    Update graph $G_{sim}$\tcp*[r]{(\Cref{eq:update})}
    
}
Get final ranking: $D_t \gets \mathrm{Graphical}(\{p_{t,i}\}_{i=1}^K,G,\alpha^*)$\tcp*[r]{(following \cite{bretz2009graphical})}
\Return{$D_t,\mathcal{F}_t$}\;
\end{algorithm2e}

\textbf{Efficient Active Sampling.}
The active evaluation sampling rule, denoted $S$, defines the selection of the next policy to evaluate and also encodes the stopping criterion. Importantly, because anytime-valid p-values (generated by the pairwise subroutine N-SCORE) are robust to measurable stopping times, the evaluation state of the underlying pairwise comparisons and the structure of the graph \emph{can be incorporated into policy selection.} For~\algname, we design the selection mechanism around the notion of \emph{advantage}, which is understood as a proxy for the value of sampling a policy with greater frequency than uniformly at random. Concretely, policies are `advantaged' when they are associated with comparisons possessing high $\alpha^{[i]}$, as these comparisons are in general the most efficient to reject.

\Cref{eqn:advantage_by_policy} shows a simple method to associate each policy $j$ with its most advantaged hypothesis. Here, we denote by $\mathcal{F}_t$ the observed data thus far\footnote{For technical reasons, $\mathcal{F}_t$ will also contain the realizations of an independent source of randomness to ensure measurability of the evaluation process for stochastic (hedged) selection mechanisms.} and by $G_{sim}$ the initial graph structure. Taken together, one can reconstruct the current graph state (i.e., incorporating any partial orderings) at time $t$. Each active hypothesis $\Hyp_i$ compares two robot policies, indexed as $\pi^{i}_A$ and $\pi^{i}_B$. We denote the number of evaluations of each policy with $n_{(.)}$. As N-SCORE evaluates an atomic hypothesis only at the lesser number of evaluations of each constituent policy, the `relevant' policy for a given hypothesis corresponds to the one with fewer evaluations thus far. Then, the advantage of a policy $j$ is equal to the largest active $\alpha^{[i]}$ for which policy $j$ is the \emph{binding} constraint:
\begin{equation}
    \label{eqn:advantage_by_policy}
    \hat{p}(j; \mathcal{F}_t, G_{sim}) := \max_i \alpha^{[i]}_t \cdot \bigg(\mathbbm{1}[\pi^{i}_A = j \text{ }\cap \text{ }n_{\pi^{i}_A} \leq n_{\pi^{i}_B}] + \mathbbm{1}[\pi^{i}_B = j \text{ }\cap \text{ }n_{\pi^{i}_B} \leq n_{\pi^{i}_A}]\bigg). 
\end{equation}
That is, policy \(j\) is a constituent of an unrejected hypothesis \(\Hyp_i\) and is the one with fewer samples. We then use a hedged sampling rule based on $\hat{p}$ to select the next policy:  
\begin{equation}
    \label{eqn:hedged_advantage_following}
    S(\mathcal{F}_t, G_{sim}) = 
    \begin{cases}
        j \propto \exp{(\hat{p}(j))}, \hspace{17mm}\text{if } \max_{j\in \{1, 2, ..., M\}} \hat{p}(j) > 0 \\
        \{\emptyset\}, \hspace{32.9mm} \text{if } \max_{j\in \{1, 2, ..., M\}} \hat{p}(j) = 0.
    \end{cases}
\end{equation}


\textbf{Sequential Graphical Testing.}
\label{method:proving_fwer_control}
Combining sequential and data-adaptive policy selection for evaluation and batch-only graphical procedures for multi-hypothesis correction is challenging because the latter is generally incompatible with data-adaptive, sequential procedures. Without careful correction, the statistical assurances that allow the risk `recycling' are invalidated. However, the robustness of anytime-valid p-values to arbitrary stopping times reconciles the conflict; 
their validity is preserved even when stopping decisions depend on simulated graph updates. By copying the true graph structure $G_{sim} \gets G$, one can use this to inform sampling (and p-value construction) without invalidating the graphical test upon completion. Crucially, the final graphical testing procedure is in fact run in a batch setting. Note, however, that this requires the initial graph structure $G$ to be fixed prior to collecting any evaluations, and cannot be modified after the evaluation period starts. Thus, by construction, \algname (\Cref{alg:multitest}) guarantees FWER control, as formalized in \Cref{theorem:validity}. We provide further validation of FWER control on synthetic data in~\Cref{subsec:BT_vs_Holm}.


\begin{theorem}[Validity of \algname]
    \label{theorem:validity}
     Given a risk budget $\alpha^* \in (0, 1)$, construct a valid graphical test $G$ over the set of pairwise hypotheses, and a measureable sampling rule $S(\mathcal{F}, G; N_{max}, \alpha^*)$ which encodes finite optional stopping of the evaluation. Then, the probability that the evaluation procedure in \Cref{alg:multitest} makes one or more incorrect ranking decisions is upper bounded by $\alpha^*$. 
\end{theorem}

\begin{remark}[Transitive closure] If we have resolved hypotheses \(\Hyp_i\) and \(\Hyp_{i'}\), respectively concluding that \(\pi_A < \pi_B\) and \(\pi_B < \pi_C\), we can by transitivity also reject hypothesis \(\Hyp_{ii''}\) that compares \(\pi_A\) and \(\pi_C\) to conclude that \(\pi_A < \pi_C\) without further evaluation, even if $p_k$ is not in a rejection range based on its N-SCORE subroutine. Evaluation with a very poor prior is one instance where this might arise in practice. We formalize this step in \Cref{cor:transitive_closure}, with the proof of validity deferred to Appendix \ref{appendix:proof_of_corollary}.  
\end{remark}

\begin{corollary}
    \label{cor:transitive_closure}
    Consider the \algname procedure defined in \Cref{alg:multitest}, modified such that the final ranking decisions are augmented with a transitive closure filter: 
    \begin{equation*}
        D_t \leftarrow \mathrm{TransitiveClosure}(\mathrm{Graphical}(\{p_{t,i}\}_{i=1}^K, G, \alpha^*)). 
    \end{equation*}
    Then, if the base instantiation controls FWER at level $\alpha^*$, the transitively closed (and, as a result, more powerful) instantiation also controls FWER at level $\alpha^*$. 
\end{corollary}

\section{Experiments}
\label{sec:evaluation}


We investigate the following research questions:
\begin{enumerate}[label=\textbf{R\arabic*}., leftmargin=*, align=left]
    \item \textbf{Statistical Reliability.} Does the~\algname~family maintain statistical reliability (i.e., FWER control), even when specified with incorrect simulation priors? Moreover, does this statistical reliability guard against policy ranking failure modes observed in Bradley--Terry baselines? 
    \item \textbf{Efficiency.} Does~\algname~increase statistical power by improving on both the number of hypotheses rejected and total sample complexity compared to other FWER-controlling methods? 
    \item \textbf{Impact of Priors.} Does the sample complexity of~\algname improve with better priors?
    \item \textbf{Scalability.} How does \algname scale as the number of policies increase?
\end{enumerate}

\textbf{Experiment Setup.} We consider three sets of hardware policy evaluations: i) RoboArena evaluations comprising crowd-sourced evaluations for generalist policies in diverse environments and tasks ~\cite{atreya2025roboarena} with pairwise preferences, ii) Large Behavior Models (\textbf{LBM}) evaluations~\cite{barreiros2026careful}, and iii) task-specific hardware policy evaluations on DROID setup (termed as \textbf{DROID-task}). In total, we consider over 7000 policy rollouts across 20+ policies. We use both world model rollouts and heldout evaluations as priors for \algname, $\algname_{\infty}$, and weighted Bonferroni methods. Additionally, we use synthetic data to study scalability of \algname, along with additional analysis on FWER control in~\Cref{subsec:BT_vs_Holm}. Unless otherwise noted, we use the global confidence level $\alpha^* = 0.1$. 

\begin{table}[t!]
\centering
\resizebox{\columnwidth}{!}{%
\begin{tabular}{c|l|c|c|c}
\hline 
& \shortstack[c]{\textbf{Methods}\\~}
& \shortstack[c]{RoboArena4 \\ \emph{{Sample complexity: 1200}}}
& \shortstack[c]{RoboArena7\\ \emph{{Sample complexity: 5106}}}
& \shortstack[c]{RoboArena7 \\ \emph{{Sample complexity: 2450}}}
\\ 
\hline 
\multirow{3}{*}{\raisebox{-5pt}{\rotatebox{90}{\textbf{Ours}}}}
& \algname
& {\(\A > \F > \G\)}, \partialtxt{\(\{\A\E\}\)}, \partialtxt{\(\{\E\F\}\)}
& {\(\partialtxt{\{\A \B \C \D \E\}} > \F > \G\)}
& \(\partialtxt{\{\A \B \C \}} > \F > \G, \partialtxt{\{\A\B\C\D\E\}}, \partialtxt{\{\D\E\F\}}\) \\
& $\algname_{0}$
& {\(\A > \F > \G\)}, \partialtxt{\(\{\A\E\}\)}, \partialtxt{\(\{\E\F\}\)}
& {\(\partialtxt{\{\A \B \C \D \E\}} > \F > \G\)}
& \(\partialtxt{\{\A \B \C \}} > \F > \G, \partialtxt{\{\A\B\C\D\E\}}, \partialtxt{\{\D\E\F\}}\) \\
& $\algname_{\infty}$
& \(\A > \E > \G\), \(\partialtxt{\{\A\F\}}\), \(\partialtxt{\{\E\F\}}\)
& \(\A > \E > \F > \G\), \(\partialtxt{\{\A \B \C \D\}}\), \(\partialtxt{\{\B\C \D \E\}}\), \(\partialtxt{\{\B \C \D \F\}}\) 
&  \(\A > \E > \F > \G\), \(\partialtxt{\{\A \B \C \D\}}\), \(\partialtxt{\{\B \C \D \E\}}\), \(\partialtxt{\{\B \C \D \F\}}\) \\
\hline
\multirow{6}{*}{\raisebox{-25pt}{\rotatebox{90}{\textbf{Baselines}}}}
& Bonferroni
& {\(\A > \F > \G\)}, \partialtxt{\(\{\A\E\}\)}, \partialtxt{\(\{\E\F\}\)}
& {\(\partialtxt{\{\A \B \C \D \E\}} > \F > \G\)}
& \(\partialtxt{\{\A \B \C \}} > \F > \G, \partialtxt{\{\A\B\C\D\E\}}, \partialtxt{\{\D\E\F\}}\) \\
& Weighted Bonferroni
& {\(\A > \F > \G\)}, \partialtxt{\(\{\A\E\}\)}, \partialtxt{\(\{\E\F\}\)}
& {\(\partialtxt{\{\A \B \C \D \E\}} > \F > \G\)}
& \(\partialtxt{\{\A \B \C \}} > \F > \G, \partialtxt{\{\A\B\C\D\E\}}, \partialtxt{\{\D\E\F\}}\) \\
& BT (Classical)
& \(\textcolor{red}{\E > \A} > \F > \G\)
& \(\A >  \textcolor{orange}{\E > \D > \B > \C} > \F > \G\)
& \(\A >  \textcolor{orange}{\D > \E > \C > \B} > \F > \G\) \\
&  BT-Davidson (leaderboard)
& \(\textcolor{red}{\E > \A} > \F > \G\)
& \(\A > \textcolor{orange}{\E > \D > \B > \C} > \F > \G\) 
& \(\A > \textcolor{orange}{ \D > \E > \C > \B} > \F > \G\) \\
&  BT-EM
& {\(\A > \E > \F > \G\)}
& {\(\textcolor{orange}{\A > \D > \B > \C} > \E > \F > \G\)}
& \(\textcolor{orange}{\D > \A > \C > \B} > {\E > \F} > \G\) \\
&  BT-EM w/ partial progress 
& \(\A > \E > \F > \G\)
& {\(\A > \textcolor{orange}{\D > \B > \C} > \textcolor{red}{\F > \E} > \G\)}
& $\A > \textcolor{orange}{ \C > \E > \D > \B} > \F > \G$ \\
\hline
\end{tabular}%
}
\caption{Policy rankings for RoboArena~\cite{atreya2025roboarena} using world model evaluations as a prior ($r=0.05$). In descending order of empirical mean progress, the policies are: \A: {\(\pi_0\)-FAST} (\(\hat{\mu} = 0.459\)), \B: {PG-FAST} (\(\hat{\mu} = 0.430\)), \C: {PG-FAST-Spec} (\(\hat{\mu} = 0.428\)), \D: {PG-FSQ} (\(\hat{\mu} = 0.419\)), \E: {PG-Diff} (\(\hat{\mu} = 0.403\)), \F: \(\pi_0\) (\(\hat{\mu} = 0.350\)), and \G: {PG-Bin} (\(\hat{\mu} = 0.051\)). Policies grouped in \partialtxt{gold} are not statistically distinguishable by the method. \textcolor{red}{Red} implies incorrect rankings, and \textcolor{orange}{Orange} indicates temporally inconsistent rankings. }
\label{tab:roboarena_rankings}
\end{table}

\noindent \textbf{Baselines.} 
Bradley--Terry (BT) is a probabilistic modeling approach that scores policies based on pairwise comparison of preferences using maximum likelihood estimation~\cite{bradley1952rank}, and policies are then ranked on the estimated scores. RoboArena presents three variants which we compare against: i) BT-Davidson which accommodates ties with an additional scalar parameter and is the current leaderboard implementation, ii) BT-EM which extends the classical model by adding task effect parameters, and iii) BT-EM w/ partial progress which  
uses partial task progress scores in addition to pairwise preferences. None of these methods guarantee FWER control. Additionally, we compare against sequentialized uniform and weighted Bonferroni corrections that control FWER by distributing the total risk budget across individual hypothesis tests and applying a union bound. In the weighted Bonferroni baseline, the initial risk is allocated in proportion to the difficulty of comparison informed by the prior; policies with closer prior performances are given a larger fraction of the budget.

\noindent  \textbf{Metrics.} An optimal multi-policy ranking procedure will return a correct and complete ranking in as few trials as possible. First, we check whether a ranking procedure maintains \textit{FWER Control}, which requires that rankings to be correct at any stopping time with high probability. Failed FWER control manifests in two key ways: i) \textit{miscalibration}, where the procedure returns a ranking decision before sufficient evidence has accumulated, and ii) \textit{inconsistency}, where the ranking procedure changes the ordering of at least two policies as more data is collected. Second, for methods that control FWER, we measure statistical power with the number of pairwise \emph{hypotheses rejected} and \emph{sample complexity}. 

\begin{figure}[h]
    \centering
    \begin{subfigure}[b]{0.32\linewidth}
        \centering
        \includegraphics[width=\textwidth]{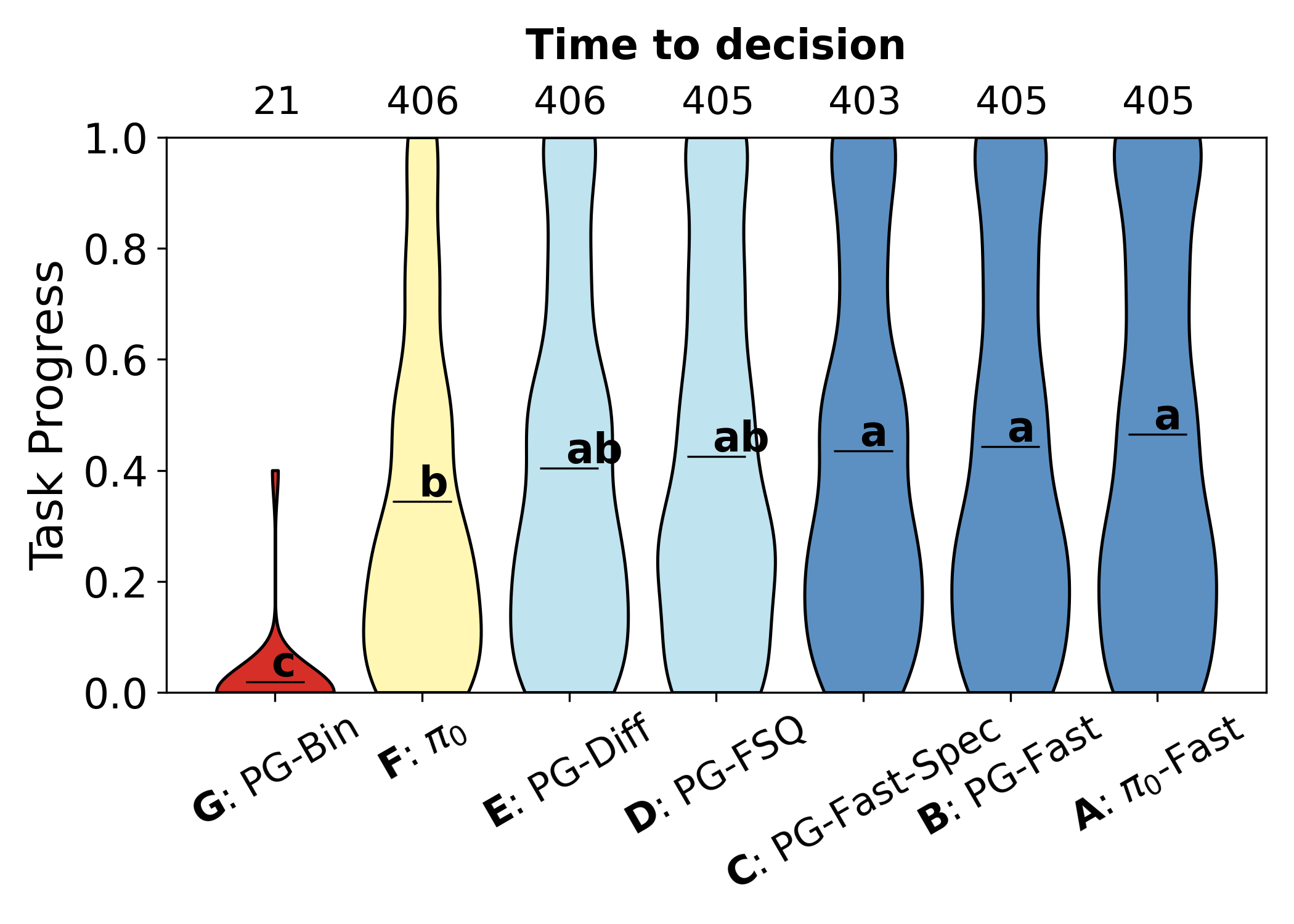}
        \caption{2450 total evaluations.}
        \label{fig:active_half_a}
    \end{subfigure}
    \hfill
    \begin{subfigure}[b]{0.32\linewidth}
        \centering
        \includegraphics[width=\textwidth]{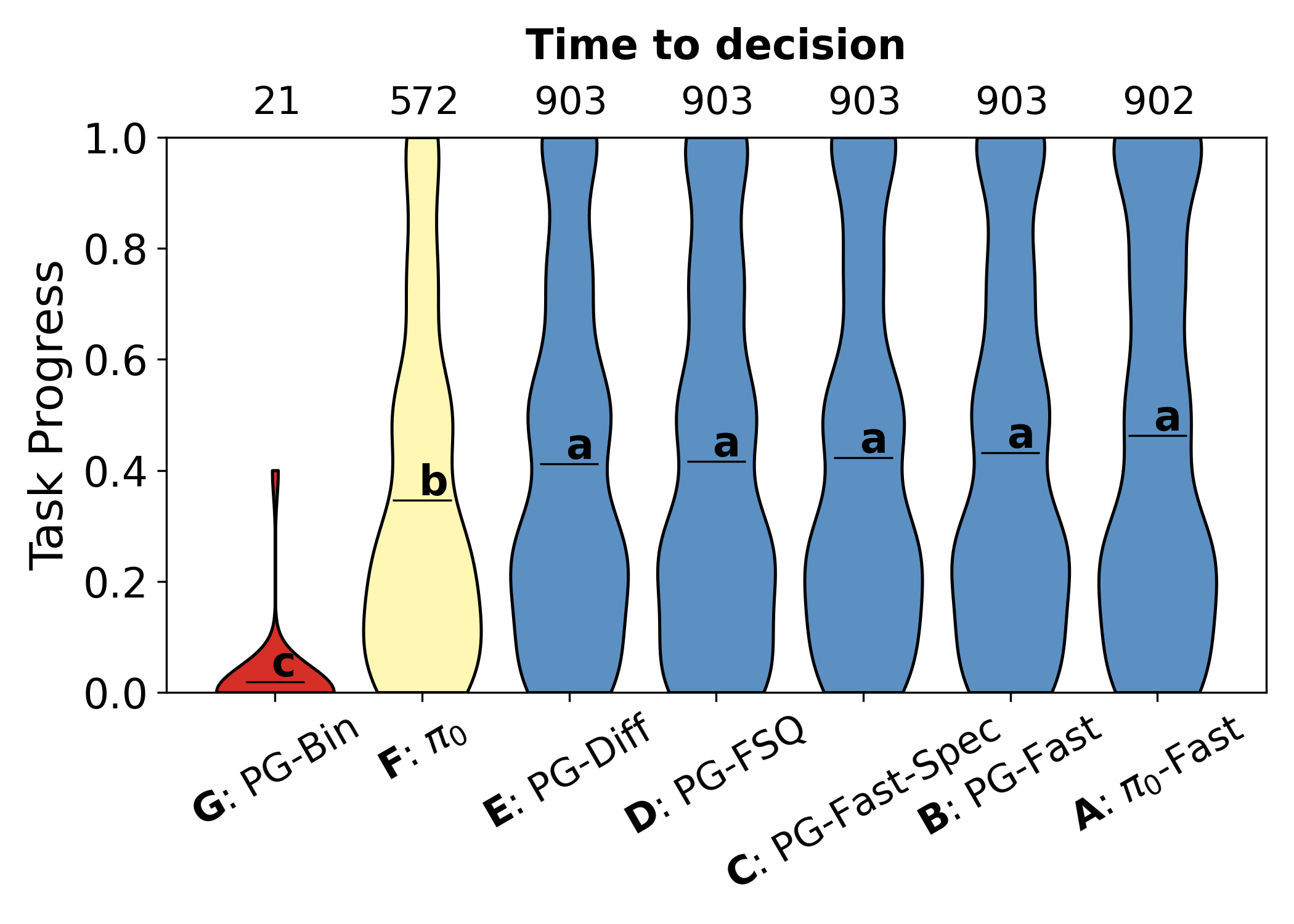}
        \caption{5106 total evaluations.}
        \label{fig:active_half_b}
    \end{subfigure}
    \hfill
    \begin{subfigure}[b]{0.32\linewidth}
        \centering
        \includegraphics[width=\textwidth]{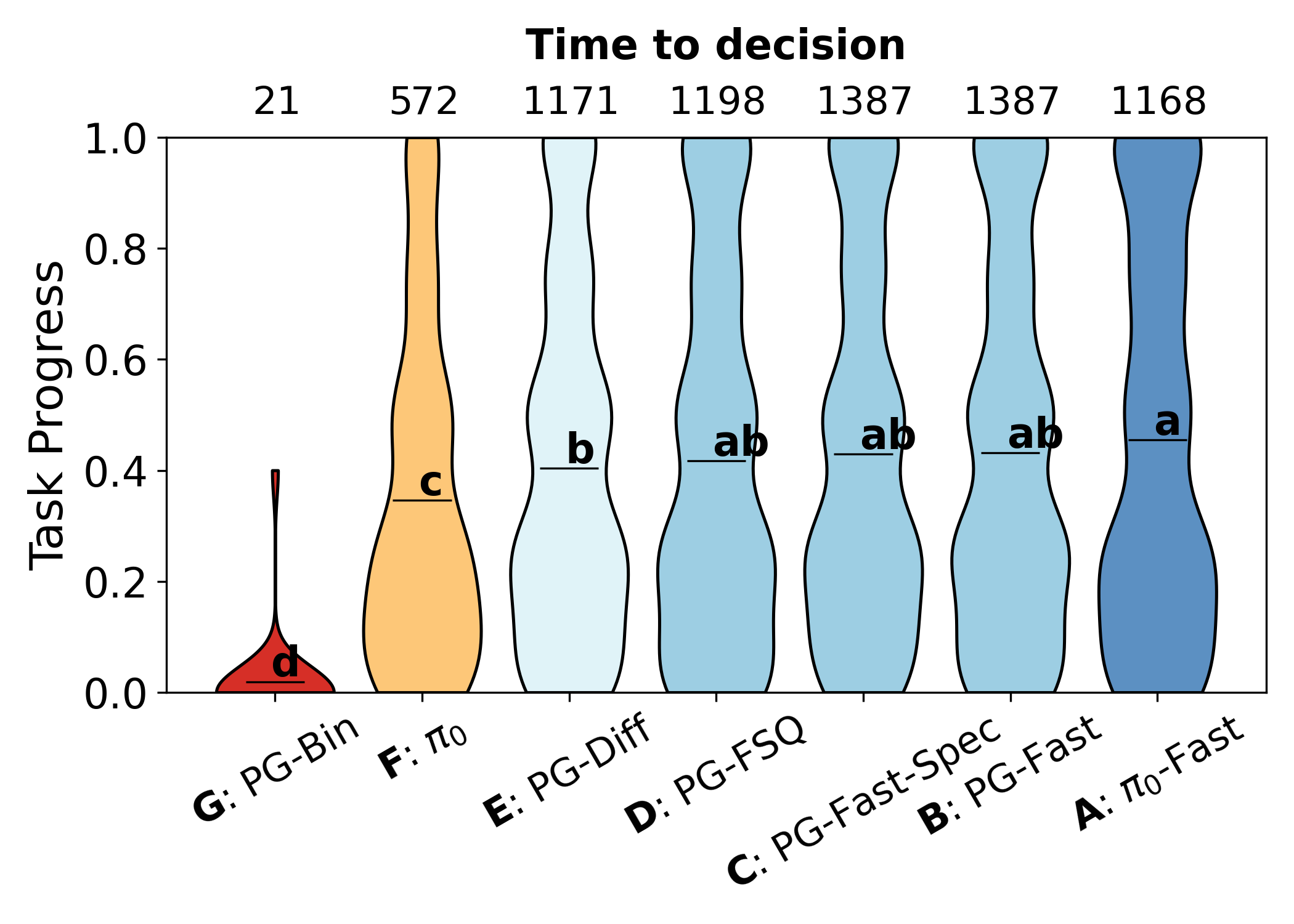}
        \caption{All available evaluations.}
        \label{fig:active_half_c}
    \end{subfigure}
    \caption{\textbf{Consistency}: \algname returns self-consistent rankings as the total sample complexity increases. Time to decision indicates the number of evaluations required for each policy. }
    \label{fig:active_half}
\end{figure}
\textbf{Reliability of Policy Rankings (R1).} 
\Cref{tab:roboarena_rankings} lists Bradley--Terry policy rankings on RoboArena, along with rankings produced by FWER-controlling methods limited to the same number of total trials as BT methods\footnote{Each RoboArena scene can have multiple policy evaluations, but only has one set of pairwise   preferences}. Unlike BT methods that always produce a complete ranking, FWER-controlling methods return partial rankings if one or more hypotheses cannot be rejected (see~\Cref{tab:roboarena_rankings}). This is often due to the number of evaluations being insufficient in relation to the difference between the true means of the policies being considered. \algname (and its ablations) are tunably \emph{robust} and temporally \emph{consistent}: they will not mislead the evaluator when the evaluation dataset is small or when there are many policies present, and will not switch ordering as the amount of data increases.~\Cref{fig:active_half} visualizes \algname's consistent policy rankings for RoboArena-7 with world model evaluations as prior using violin plots to depict the empirical spread of raw evaluation data and compact letters~\cite{piepho2004algorithm} to denote ranking assignments. If two policies share a ranking letter assignment, then given the evaluations, they are not statistically distinguishable with \(90\%\) confidence. 

In contrast, BT methods do not provide FWER control. In RoboArena4 (\Cref{tab:roboarena_rankings}), Classical and Davidson BT methods report incorrect rankings with high probability (\emph{miscalibration}), scoring policy \(\E\) higher than policy \(\A\) (\emph{miscalibration}). Given that \algname, using all available data, correctly ranks the policies with \(90\%\) probability (see~\Cref{tab:roboarena_rankings_full} and~\Cref{fig:active_half_c}), we can state that there is a \(90\%\) chance that BT ranking of \(\E > \A\) is incorrect. Secondly, BT rankings need not be self-consistent as more evaluations are collected (\emph{inconsistency}). This can be seen in~\Cref{tab:roboarena_rankings}, where all BT variants are self-inconsistent in RoboArena7. A second instance of inconsistency is when the number of policies are increased: BT (EM w/ partial progress) incorrectly concludes \(\F > \E\) in Roboarena7 (at 5106 evaluations) despite using the same policy evaluation data for \(\F\) and \(\E\) as in RoboArena4. 

Furthermore, \algname rankings in~\Cref{tab:roboarena_rankings} use world model evaluations as a prior, which has low correlation with real-world empirical policy rankings for policies \(\A-\F\) ($r=0.05$; ~\Cref{fig:roboarena_sim}). Despite these imperfect priors, \algname policy ranking is still guaranteed to be correct with high probability, and the sample complexity remains comparable to the no-prior ablation \(\algname_0\). 
 
\begin{wrapfigure}{r}{0.32\textwidth}
    \centering
    \includegraphics[width=0.85\linewidth]{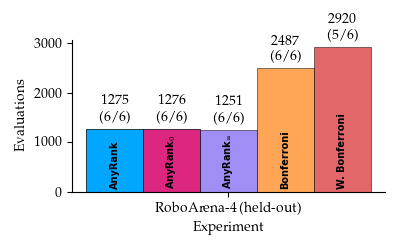}
    \caption{
       RoboArena4 with heldout evaluations prior.
    }
    \label{fig:roboarena_samples}
    \vspace{-10mm}
\end{wrapfigure}
\textbf{Policy Evaluation Efficiency (R2).}
In \Cref{fig:roboarena_main,fig:tri_ood}, we see that \algname and $\algname_0$ reject the largest number of hypotheses with in the fewest number of samples. On RoboArena4 (with heldout evaluations as a prior), \textbf{the \(\algname\) family achieves around a \(50\%\) reduction in hardware trials to fully decide compared to Bonferroni} (\Cref{fig:roboarena_samples}). Similar efficiency gains of \algname over Bonferroni hold on RoboArena and TRI-LBM (\Cref{fig:roboarena_main,fig:tri_ood}). 
Weighted Bonferroni rejecting the fewest hypotheses despite having the highest sample complexity. This efficiency of \algname over Bonferroni is enabled by the risk re-allocation of the underlying graphical testing methods.

\begin{figure}[h]
    \centering
    \begin{subfigure}[t]{0.3\linewidth}
        \centering
        \includegraphics[width=\linewidth]{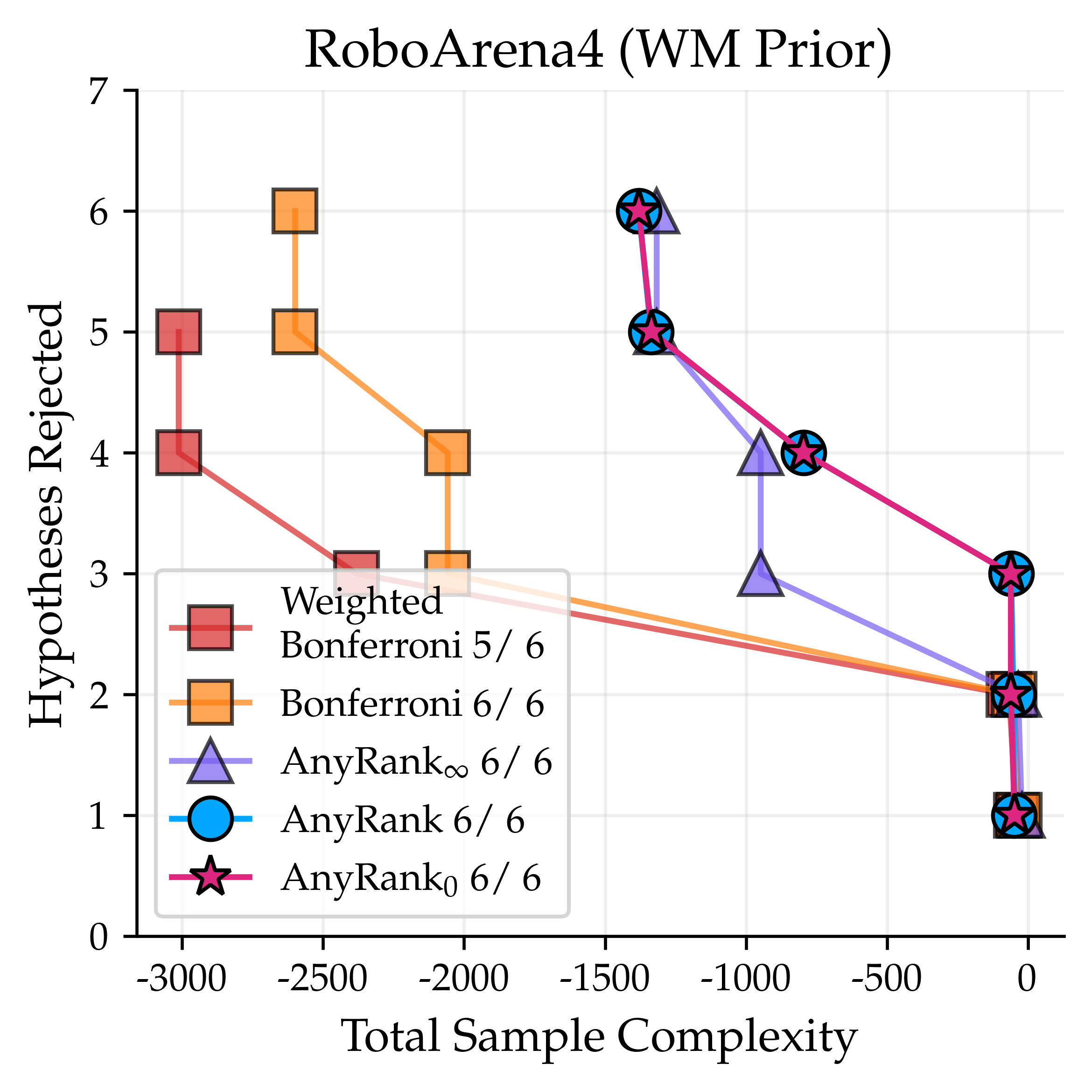}
    \end{subfigure}
    \hfill
    \begin{subfigure}[t]{0.3\linewidth}
        \centering
        \includegraphics[width=\linewidth]{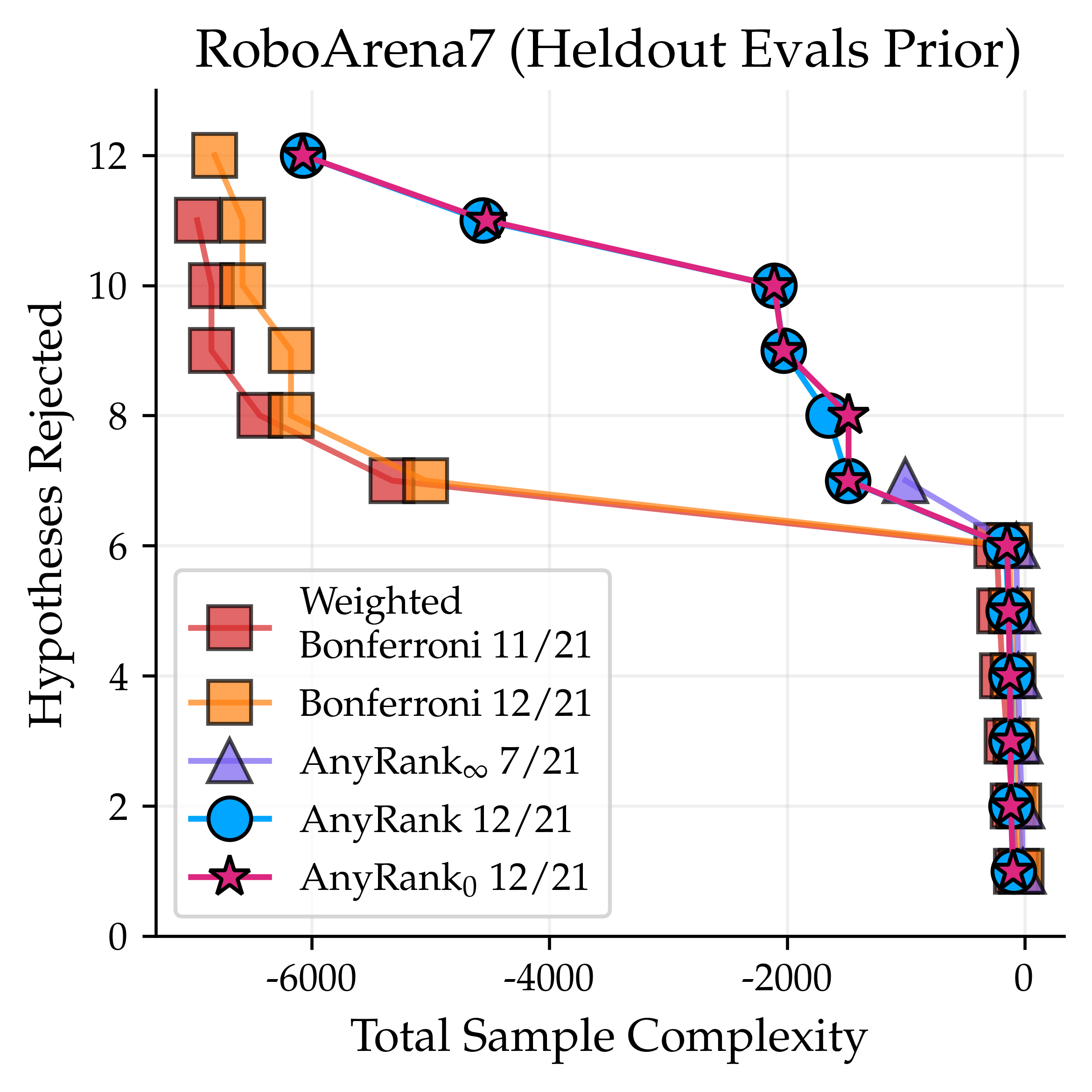}
    \end{subfigure}
    \hfill
    \begin{subfigure}[t]{0.3\linewidth}
        \centering
        \includegraphics[width=\linewidth]{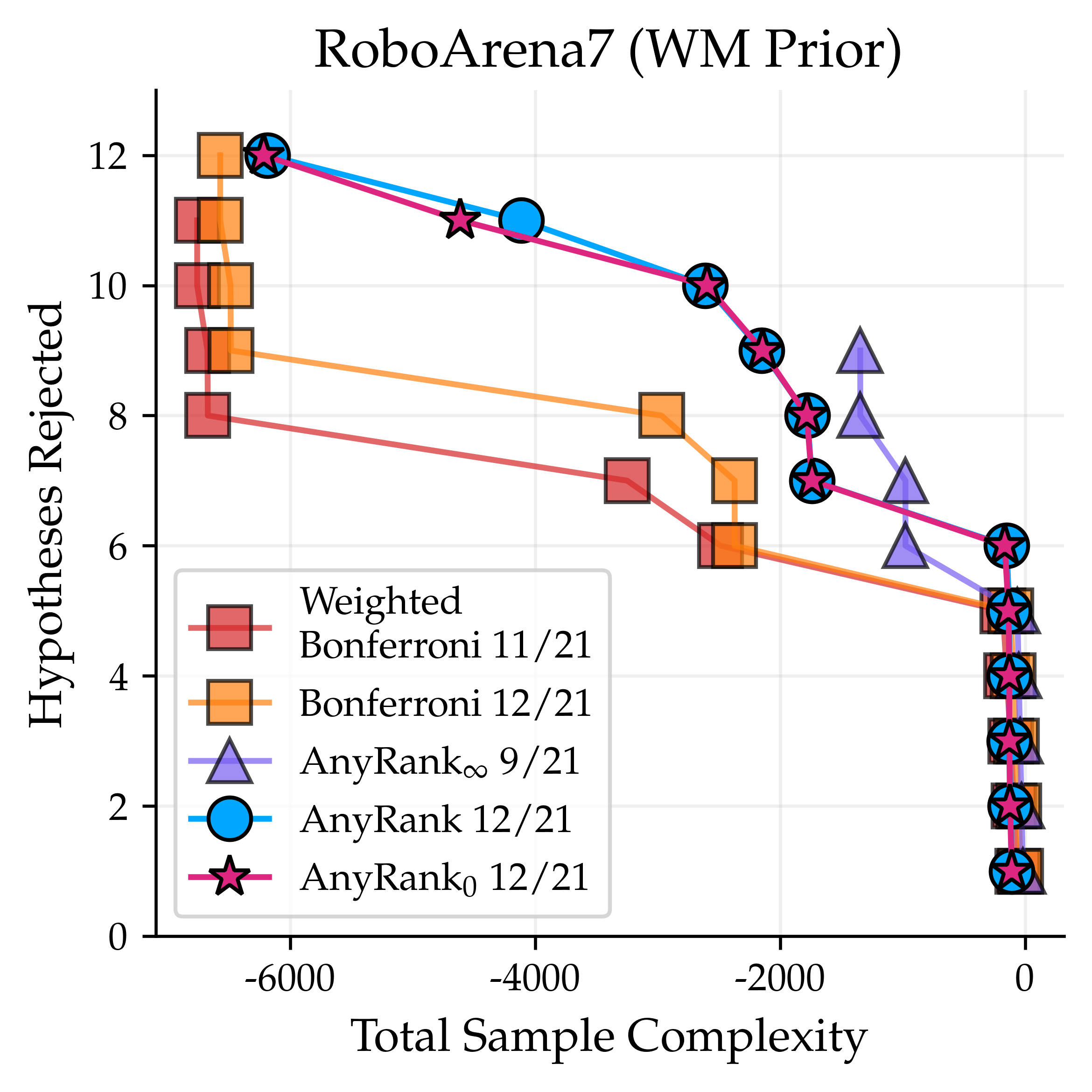}
    \end{subfigure}
    \caption{\textbf{Pareto Curves for RoboArena.} Efficacy of ranking methods by the number of atomic hypotheses rejected as well as the (negated) total sample complexity. The upper right of the plot corresponds to the most effective method: maximizing hypotheses rejected with low sample complexity indicates greater progress towards achieving a full policy ranking in the fewest hardware evaluations.}
    \label{fig:roboarena_main}
\end{figure}

\begin{figure}[h]
    \centering
    \begin{subfigure}[t]{0.3\linewidth}
        \centering
        \includegraphics[width=\linewidth]{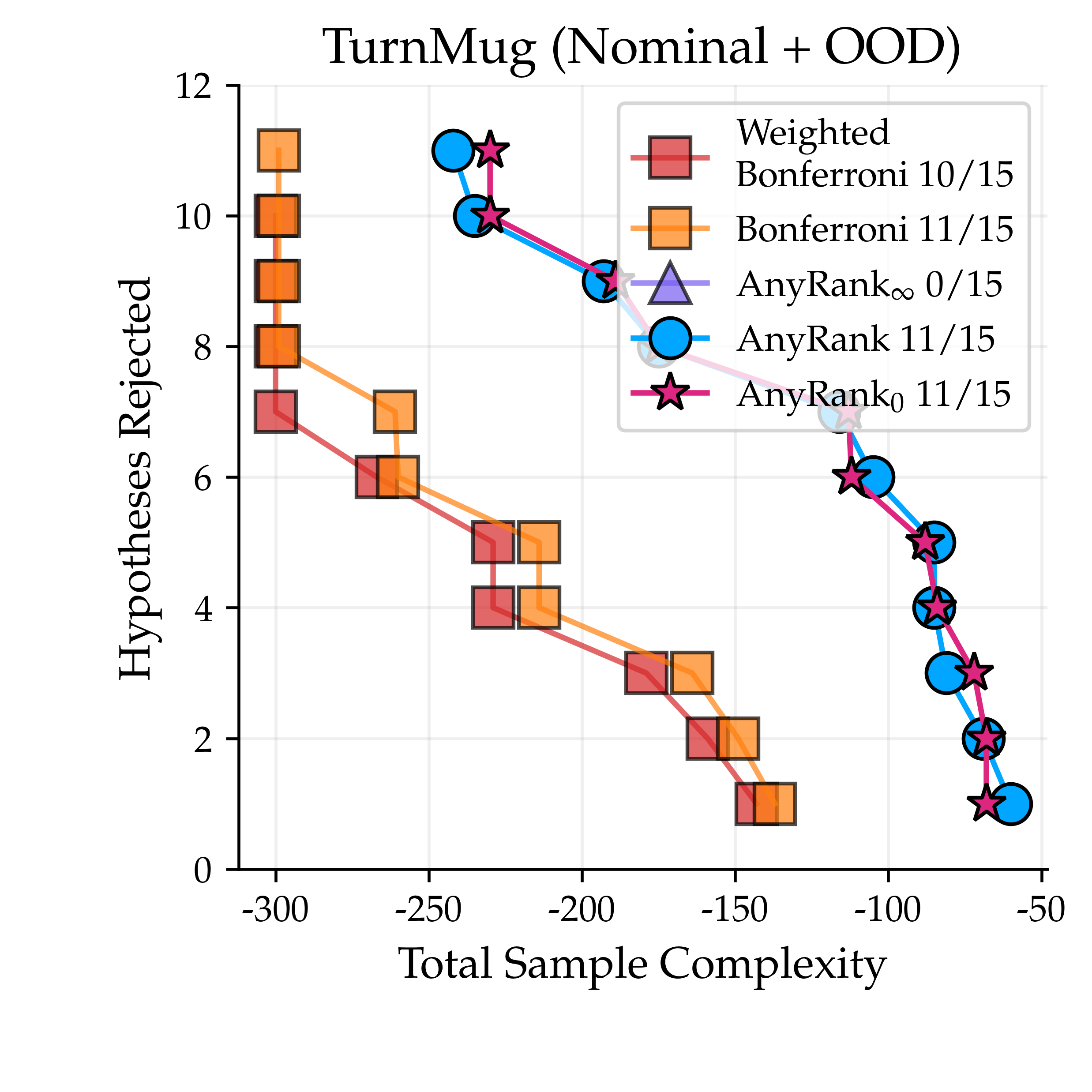}
        \caption{TurnMug (Nominal and OOD)}
    \end{subfigure}
    \hfill
    \begin{subfigure}[t]{0.3\linewidth}
        \centering
        \includegraphics[width=\linewidth]{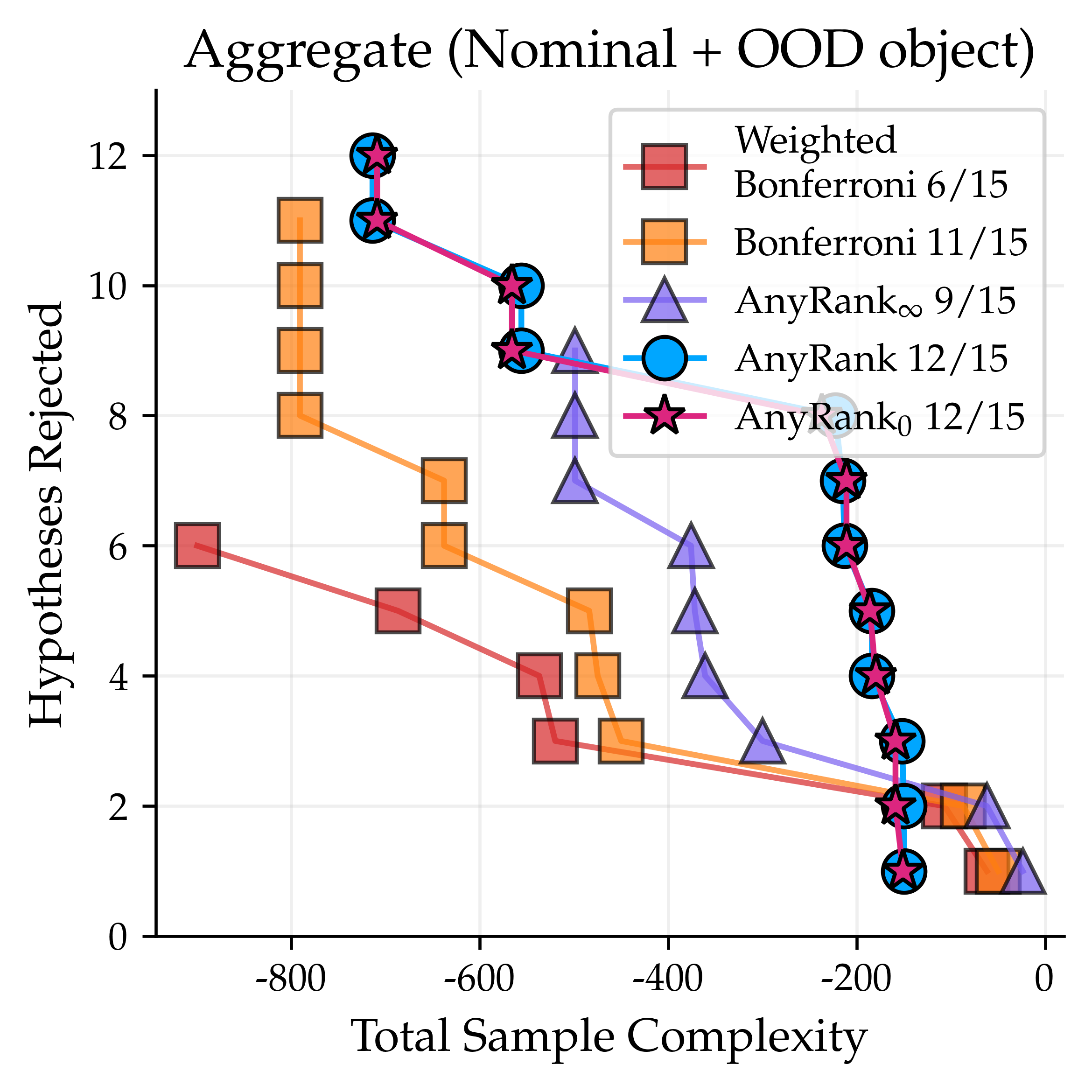}
        \caption{Aggregate: OOD Object}
    \end{subfigure}
    \hfill
    \begin{subfigure}[t]{0.3\linewidth}
        \centering
        \includegraphics[width=\linewidth]{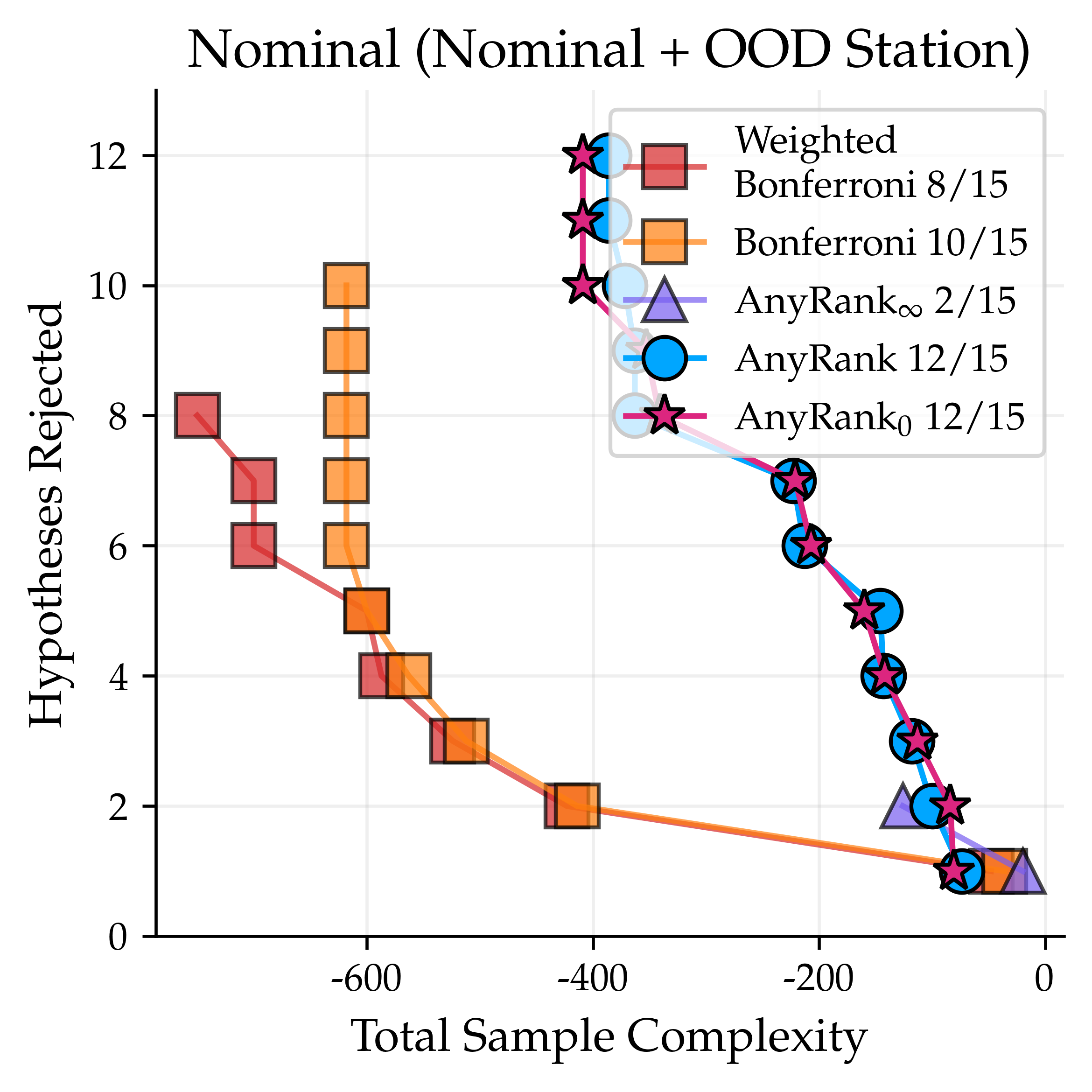}
        \caption{Aggregate: OOD Station}
    \end{subfigure}
    \caption{\textbf{Pareto Curves for TRI-LBM.} \algname is the most efficient, rejecting the largest number of hypotheses in the fewest number of samples, comparable to the no prior variant \(\algname_0\).}
    \label{fig:tri_ood}
\end{figure}

\begin{figure}[h]
    \centering
    \includegraphics[width=\linewidth]{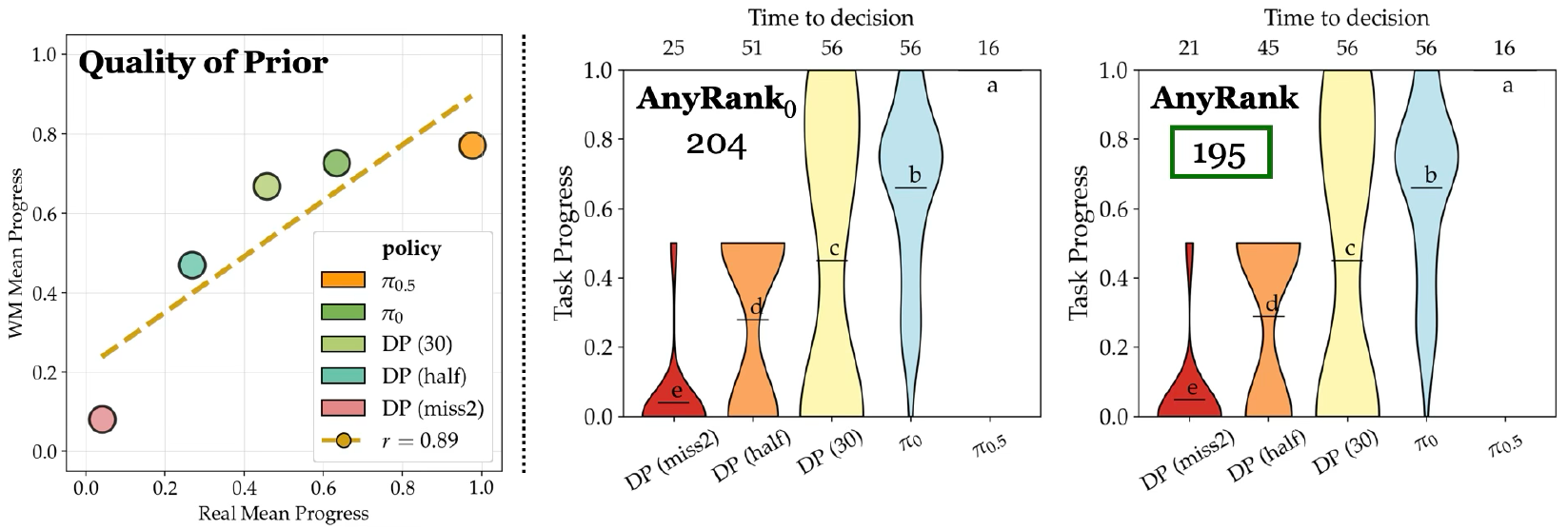}
    \caption{
    \textbf{R4. \algname improves with good priors.} 
     \algname (\(\gamma=10\)) achieves full policy ranking with fewer evaluations compared to~$\algname_{0}$.
    }
    \label{fig:good_prior}
\end{figure}
When all hypotheses are decidable (e.g., RoboArena4), \(\algname_{\infty}\) is slightly more efficient in total sample complexity after rejecting all hypotheses as it sequentially allocates the full risk budget to one hypothesis at a time in order of hypothesis difficulty as specified by the prior.  In general, however, \(\algname_{\infty}\) can get stuck on difficult-to-reject hypotheses such as in RoboArena7 (\Cref{fig:roboarena_main}) and TRI-LBM (\Cref{fig:tri_ood}), resulting in fewer hypotheses rejected. \algname provides a suitable trade-off: it always rejects a larger fraction of hypotheses by adaptively reallocating confidence budget during evaluation, shifting the budget toward harder hypotheses based on real-world evaluation data. 

\textbf{AnyRank Improves with Good Priors (R3).}
To investigate whether stronger priors improve sample efficiency of \algname over $\algname_0$, we consider a ``Fold towel" task on the DROID platform, and fine-tune a video world model using play data~\cite{yin2026playworld} to get strong policy ranking signals ($r=0.89$). In~\Cref{fig:good_prior}, we see that \algname requires fewer samples than $\algname_{0}$ (204 vs. 195). Additionally, $\algname_{\infty}$ can reject all hypotheses in \(188\) trials, a lower bound on the minimum required hardware evaluations. These results show that while strong priors can are helpful, the efficiency gains over no prior graphical versions need not be substantial. 

\begin{figure}
    \centering
    \includegraphics[width=\linewidth]{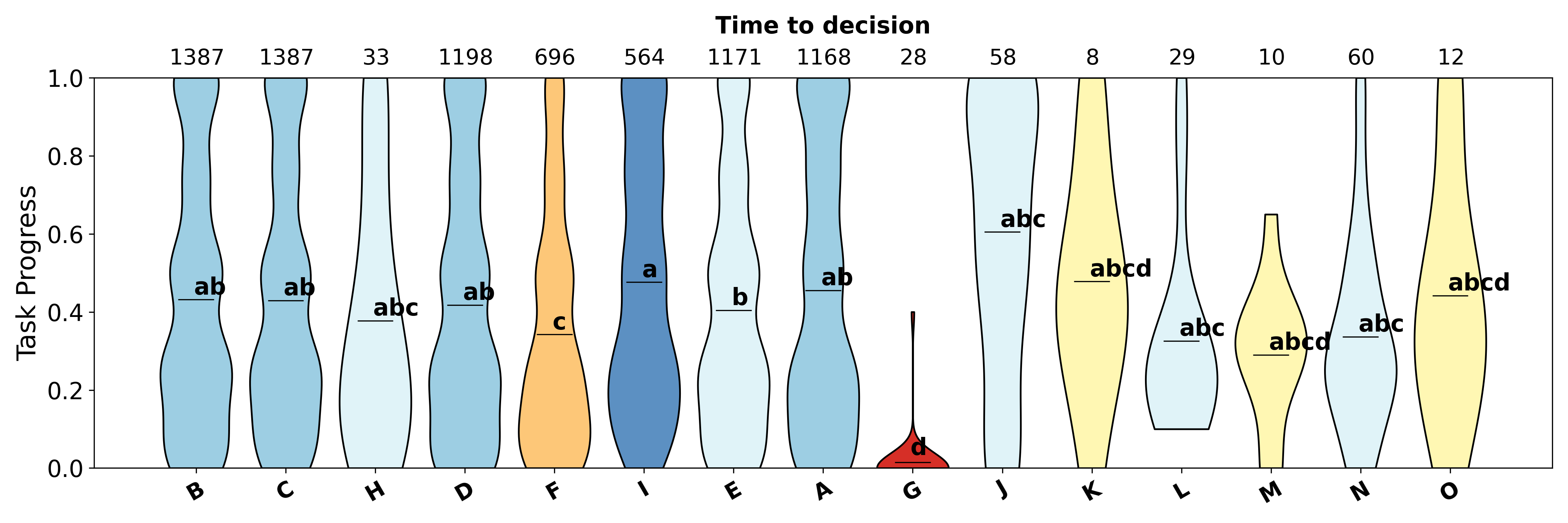}
    \caption{\algname ranking 15 RoboArena policies using all available data. Bradley--Terry rankings for these policies are detailed in~\Cref{tab:roboarena_bt_rankings}.}
    \label{fig:scalability}
\end{figure}

\textbf{Scalability (R4).}
 Risk re-allocation of the sequential graphical framework enables~\algname~ to attenuate the \(\mathcal{O}(M^2)\) loss of power associated with Bonferroni methods. However, as the number of policies increases, more policy pairs become difficult to distinguish, increasing the inherent difficulty of ranking for \emph{any statistically valid method}. As shown in the experiments on synthetic data in \Cref{Sample_Complexity}, scaling the number of policies increases the total sample complexity but shows uniform advantage of \(\algname\) over standard Bonferroni.

\begin{table}
    \centering
        \resizebox{\linewidth}{!}{%
        \begin{tabular}{cccccc}
        \toprule
        $N_{\text{max}} \downarrow$ &
        \# of policies $M \rightarrow$ & 5 & 10 & 15 & 20 \\
        \midrule
        500 & \algname & \textbf{1102} (9/10) & \textbf{2480} (44/45) & \textbf{6705} (98/105) & \textbf{9075} (182/190) \\
        500 & Bonferroni & 1170 (9/10) & 2855 (43/45) & 6742 (93/105) & 9220 (175/190) \\
        \midrule
        5,000 & \algname & \textbf{1554} (10/10) & \textbf{3259} (45/45) & \textbf{38k} (102/105) & \textbf{57k} (185/190) \\
        5,000 & Bonferroni & 1683 (10/10) & 5191 (45/45) & $>42$k (102/105) & $>63$k (185/190) \\
        \bottomrule
        \end{tabular}
        }
    \caption{\small Sample complexity and fraction of hypotheses rejected on synthetic data at significance \(\alpha^*=0.1\) and \(N_{\max}\in\{500,5000\}\) per policy.}
    \label{Sample_Complexity}
\end{table}
Finally, to illustrate scalability on real data, we rank 15 policies from RoboArena, including some which were subsequently retracted from the leaderboard due to insufficient evaluations. Again, we observe that Bradley--Terry rankings are susceptible to miscalibration and inconsistency, especially at small sample sizes. For example, in~\Cref{tab:roboarena_bt_rankings}, Bradley-Terry (Davidson) outputs a full ranking, but according to~\Cref{fig:scalability}, a lot of the policies remain statistically indistinguishable due to small sample sizes relative to differences in mean performance. While Bradley-Terry ranks some of these few-trial policies highly (e.g., policies \textbf{H} and \textbf{K} in~\Cref{tab:roboarena_bt_rankings}), \algname is robust to the limited evidence, returning a partial ranking that reflects the statistical uncertainty.

\section{Conclusion}
\label{sec:conclusion}
We propose \algname, a policy ranking framework that produces anytime-valid rankings, provides statistical assurances on their correctness, and achieves sample efficiency by adaptively allocating trials in a data-dependent manner and optionally leveraging prior information. This is made possible through novel theoretical contributions that fuse anytime-valid inference with multi-hypothesis testing. Experiments across real-world benchmark datasets demonstrate significant improvements in statistical reliability and sample efficiency compared to state-of-the-art approaches. 

{\bf Limitations and future work.} Our framework assumes \(i.i.d\) evaluation data for providing statistical assurances. Evaluation best practices~\cite{kress2024robot,barreiros2026careful} are particularly effective in this regard, some of which are incorporated into~\algname, such as adaptive sampling that interleaves policy evaluations. In addition, the number of hardware evaluations required for statistically reliable policy comparison can be large~\cite{snyder2025your}. One approach to addressing this limitation is the use of proxy evaluations as a substitute for real evaluations~\cite{badithela_reliable_2025}, which we leave to future work.


\section*{Acknowledgments}
We thank Lihua Lei for insightful discussions on graphical testing and safe anytime valid inference. A. Badithela is supported by the Presidential Postdoctoral Fellowship. D. Snyder acknowledges support from the Toyota Research Institute (TRI) and Penn AI Fellowship. The authors were also partially supported by the NSF Career Award \#2044149, the NSF SLES Award \#2331880, and the Sloan Fellowship. TRI provided funds to assist the authors with their research; this work solely reflects the opinions and conclusions of its authors, and not TRI nor any other Toyota entity.


\bibliographystyle{unsrtnat}
\bibliography{references}

\clearpage

\beginappendix{
    \appendix 
\section*{\huge Supplementary Material}

\section{Technical Implementation}

\subsection{Strategies for Graph Propagation and Risk Reallocation}
\label{sec:realloc}
We motivate the intuition for prior-dependent graph initialization using as illustrative comparison two special cases of multitest procedures: the Holm procedure \cite{holm_1979} and the fixed-sequence procedure \cite{bretz2009graphical}. The Holm structure corresponds to uniform initial risk allocation and reallocation weights. It can be shown~\cite{holm_1979} that for $K$ hypotheses, the associated cutoffs at the time of decision for each rejected atomic hypothesis will be (in ordered form) $\{\alpha^{[(n)]}\} \equiv \{\alpha^* / n: n \in \{K, ..., 1\}\}$. This strictly dominates Bonferroni, as the total effective risk budget becomes $\sum_n \alpha^{[(n)]} = \sum_{n} (\alpha^* / n) > \alpha^*$. As a concrete example, \Cref{fig:graph_propagation} illustrates the Holm structure, with cutoffs $\{\alpha^{[i]}\} = \{0.10, 0.15, 0.30\} = \{\alpha^* / i: i \in \{3, 2, 1\}\}$. In contrast, a fixed-sequence test corresponds to a chain graph in which the entire risk budget is initially assigned to a single hypothesis, which is then propagated to a single new hypothesis when a decision is made. When every hypothesis is decidable, this is the most efficient method, as the cutoffs at decision will be the least restrictive: $\alpha^{[i]} = \alpha^*\text{ }\forall i$, exactly matching the theoretical upper bound. However, this approach is very committal -- when not every hypothesis is decidable (e.g., when two policies have the same mean or, in finite samples, when the means are very close), fixed sequence methods will get stuck, often very early, and all subsequent hypotheses will fail to decide. A Holm-like procedure thus sacrifices some efficiency per hypothesis ($\alpha^{[i]} < \alpha^*$) to hedge against becoming stuck evaluating any one very hard comparison. In general, there is a continuum of graph specifications which trade off hedging and performance, represented by the initial risk allocation and reallocation weights. Selecting these weights (near-)optimally for policy ranking is the first central concern of this work. 

\subsection{Graph Initialization using a Prior}
\label{sec:graph_instantiation}
To initialize the graphical multiple testing procedure, we leverage prior estimates of policy performance differences. Let \(\hat{\mu}\) denote empirical performance estimate (e.g., mean sucess rate or mean progress score) of a policy \(\pi\) from simulation or world model rollouts, or heldout hardware evaluations. Suppose hypothesis \(\Hyp_i\) involves the pairwise comparison of policies \(\pi^{i}_A\) and \(\pi^{i}_B\). Let \(d_i = |\hat{\mu}^{i}_A - \hat{\mu}^{i}_B|\) denote the magnitude of the predicted performance gap between the two policies compared by hypothesis \(\Hyp_i\).

A large \(d_i\) indicates a large predicted gap in policy performance. If this prediction is reflective of real-world performance, then the corresponding policies are easier to statistically distinguish, and the hypothesis \(\Hyp_i\) can typically be decided with fewer evaluation trials on hardware. Therefore, an efficient algorithm will assign large $\alpha^{[i]}$ to these hypotheses, in order that they decide quickly and can recycle the risk budget to other hypotheses. Conversely, when \(d_i\) is small, the policies are expected to exhibit similar performance, requiring more data to establish a statistically significant difference; as such, they should be assigned a smaller $\alpha^{[i]}$, from which they will accumulate recycled risk as easier evaluations terminate early.
We use the predicted performance gap to initialize both the local significance levels and the graph edge weights. Specifically, each hypothesis \(\Hyp_i\) is assigned an initial significance level and outgoing graph edge-weights as follows:
\begin{equation}
\label{eq:init}
    \alpha^{[i]}_0 = \alpha^{*} \frac{\exp(\gamma d_i)}{\sum_j \exp(\gamma d_j)},
\qquad
w^{[ij]} = \frac{\exp(\gamma d_j)}{\sum_{j\neq i} \exp(\gamma d_j)},
\end{equation}

where \(\alpha^{*}\) is the global significance level and hyperparameter \(\gamma \geq 0\) controls the strength of the prior. Larger values of \(\gamma\) place a greater fraction of the significance budget on hypotheses with larger predicted performance differences. Similarly, the initial graph edge-weights are such that successor hypotheses \(\Hyp_j\) with larger predicted performance gaps have higher weights \(w^{[ij]}\). When \(\Hyp_i\) is rejected, its local significance level \(\alpha^{[i]}\) gets redistributed to the remaining unresolved hypotheses in proportion to \(w^{[ij]}\), with larger fractions going to successors \(\Hyp_j\) that are predicted to be easier to decide. Allocating higher significance levels to these hypotheses further increases the chance that they are rejected early, thereby accelerating the redistribution of significance levels throughout the graph, and ultimately giving the hypotheses that are the most difficult to reject the highest possible significance level. 

Note that if multiple hypotheses reject at the same time, the specific order of applying graph propagation does not matter~\cite{bretz2009graphical}.
At the same time, if the prior is inaccurate, the graphical procedure still remains valid and adapts its allocation through subsequent significance-level transfers, preserving family-wise error control regardless of the quality of the prior. 

\subsection{Proof of Theorem 1}
\begin{proof}
    Given the assumptions of \Cref{theorem:validity}, the relies on isolating the decision-making of which policies to sample next during evaluation from the terminal (batch) graphical evaluation itself (see \Cref{alg:multitest}). The key intuition is that the anytime-validity of the p-values of individual hypotheses is not affected by the risk reallocation of the graph G insofar as the graph is initialized before evaluations are collected. That is, the initial risk allocation and initial graph weights are independent of any evaluations, and their reallocation is informed via a fixed graphical update rule, and not on the basis of any collected evaluations themselves. The set of anytime p-values $\{p_{t, i}\}_{i=1}^K$ is uniformly valid from the independence assumption, the measurability of $S$ (generating a stopping time), and the definition of p-processes (\cite{ramdas_admissible_2022}, Prop. 12.1). Assuming that the evaluator has a finite, yet potentially very large evaluation budget (\(N_{\max}\)) guarantees that the graphical procedure is run w.p. 1. Thus, the ranking decisions $D_t$ constitute the output of a graphical procedure, valid at $\alpha^*$, that is w.p. 1 given valid p-values for all constituent hypotheses. The validity of the graphical multi-hypothesis decisions given these inputs follows from \citet{bretz2009graphical} (c.f. their Appendix A and Algorithm 1).  
\end{proof}

\subsection{Detailed Implementation Notes for \Cref{alg:multitest}}
\label{supp:long_algorithm}
We provide additional, specific details for running~\algname~in \Cref{alg:multitest_long}. This exposition is identical in practice to the implementation of \Cref{alg:multitest}, but we more explicitly track various aspects of the filtration and data storage, as well as the mechanization of steps left implicit in the condensed version. 
\begin{algorithm2e}[htbp]
\caption{Multi-Hypothesis Active Policy Ranking with \algname}
\label{alg:multitest_long}

\KwIn{Type-I error limit $\alpha^* \in (0,1)$, policies \(\{\pi_1,\ldots \pi_M\}\), prior mean performance \(\{\hat{\mu}_1, \ldots, \hat{\mu}_M\}\)}
\KwOut{Ranking decisions $D_t$ and evaluation samples $\mathcal{F}_t$}
\emph{Initialize graph}: $G_{sim} \gets $ Set initial risk and weights using prior (\Cref{eq:init})\\
\emph{Duplicate graph}: $G \gets G_{sim}$ For the final sequential graphical analysis\\
\emph{Initialize decision list}: $D_0 \gets $ All policies equally ranked\\
\emph{Initialize evaluation step}: $t \gets 0$\\
\emph{Initialize dataset}: $\mathcal{F}_0 \gets \emptyset$\\
\emph{Initialize per-policy evaluation counts}: $n_j \gets 0,\, \forall j=1,\ldots, M$\\
Select initial policy: $j_0 \gets S(\mathcal{F}_0,G_{sim})$\

\While{$j_t \neq \emptyset$}{
    Run policy $j_t$ on environment $e_t \sim \mathcal{D}$:
    $r_{j_t}(e_t) \gets \mathrm{Rollout}(e_t,\pi_{j_t})$

    
    Update evaluation dataset:
    $\mathcal{F}_{t+1} \gets \mathcal{F}_t \cup \{(j_t, r_{j_t,t})\}$

    Update policy evaluation count:
    $n_{j_t} \gets n_{j_t}+1$

    \For{$i \in \{1,\ldots,K\}$}{
        Update anytime-valid p-value:
        $p_{t+1,i}\gets\mathrm{N\text{-}SCORE}(\mathcal{H}_i,\mathcal{F}_{t+1})
        $
    }
    \For{$i \in \{1,\ldots,K\}$}{
    \If{$p_{t+1,i} \leq \alpha^{[i]}_{t+1}$}{
        Update $G_{\mathrm{sim}}$ (\Cref{eq:update})
    }
}
    Select next policy to evaluate:
    $j_{t+1} \gets S(\mathcal{F}_{t+1},G_{sim})$ (Hedged sampling~\Cref{eqn:hedged_advantage_following})
    
    $t \gets t+1$
}
Get final ranking:
    $D_t
    \gets
    \mathrm{Graphical}(\{p_{t,i}\}_{i=1}^K,G,\alpha^*)$ (using \cite{bretz2009graphical})
    
\Return{$D_t,\mathcal{F}_t$}\
\end{algorithm2e}
\section{Proof of Validity Under Transitive Closure}
\label{appendix:proof_of_corollary}
Three key components are required to demonstrate this property. First, we formalize the ranking problem in the \emph{lifted space} of possible full rankings; second, we demonstrate that pairwise comparisons form a sufficient compression of this space so as to enable transitive closure; finally, we illustrate an equivalent decision rule in the lifted decision space that exactly recovers \algname with transitive closure. Thus, by showing that the decision rules are equivalent, we consequently demonstrate that adding transitive closure does not induce additional family-wise error. 

\subsection{The Lifted Representation}
To simplify exposition, we label each policy in a ranking problem by a letter or set of letters (e.g, ``A," ``B," ``C," ...). A comparison is explicitly represented by an inequality indicator, e.g, $\mathbbm{1}[A > B]$. A hypothesized ranking is represented by an ordered string. For example, $ABCDEF$ is understood to represent the global ranking hypothesis ``Policy A is better than Policy B, is better than Policy C, ..." The formal candidate true policy ranking is represented by a boldface string, e.g., \textbf{ABCDEF}. Testing decision rules correspond to a mapping from evidence on belief rankings (e.g., $ABCDEF$) to \textit{decisions} about candidates (e.g., \textbf{ABCDEF}). To make the exposition more concrete, we will carry through the practice of using six policies (i.e., the letters A through F); however, the argument generalizes to any number of policies, considering the rankings to essentially consist in strings of arbitrary length. 

For $K$ policies, there are 2$K\choose2$$= K^2 - K$ pairwise comparisons. There are $K!$ candidate true rankings. The important step is as follows: construct a `belief space' over rankings consisting of the set of all hypothesized rankings consistent with each pairwise comparison. For each comparison (e.g., $\mathbbm{1}[A > B]$), precisely half of the candidate true rankings (i.e., $K!/2$) are consistent with the comparison, by a substitution argument.\footnote{Specifically: every full ranking hypothesis must contain the constituent policies (here, $A$ and $B$). For every ranking consistent with the hypothesis, (e.g., \textbf{ABCDEF}), an inconsistent can be constructed by substituting $A$ and $B$ (e.g., \textbf{BACDEF}).} This will result in $K! (K^2 - K) / 2$ belief space rankings, consisting in $(K^2 - K)/2$ belief copies of each of the $K!$ candidate rankings (by a similar substitution argument, each candidate true ranking is associated with belief copies for precisely half of the pairwise hypotheses, e.g. candidate \textbf{ABCDEF} has a belief copy $ABCDEF$ associated with $\mathbbm{1}[A > B]$, but not with $\mathbbm{1}[B > A]$). 

Now, formally associate with each pairwise comparison the $K!/2$ hypothesized belief rankings with which is consistent in the form of a chain graph. For now, each chain graph can be ordered at random; in other words, the first and last elements of the chain do not matter. It is clear that, by construction, pairs of converse hypotheses (i.e., $\mathbbm{1}[A > B]$ and $\mathbbm{1}[B > A]$) fully partition a copy of the full candidate true ranking set; as there are $K\choose2$ pairs of converse hypotheses, each candidate true ranking again has $K\choose2$ belief copies. 

\subsection{Transitive Closure as Set Inclusion}
Understanding this partition idea, applying transitive closure amounts to considering relative set sizes of belief hypotheses induced by different combinations of pairwise comparisons. 

In particular, consider the variable policy indices $\alpha, \beta, \gamma$, and denote by $\Sigma_{\alpha > \beta}$ the set of belief rankings implied by the hypothesis $\mathbbm{1}[\alpha > \beta]$. Again, to be concrete, $ABCDEF \in \Sigma_{\alpha=A > \beta=B}$. We wish to show 
\begin{equation}
    (\Sigma_{\beta > \alpha} \cup \Sigma_{\gamma > \beta} ) \text{ }\cap \text{ } \Sigma_{\gamma > \alpha} = \Sigma_{\gamma > \alpha}.
\end{equation}
Equivalently: 
\begin{equation}
    \forall s \in \Sigma_{\gamma > \alpha}, s \in (\Sigma_{\beta > \alpha} \cup \Sigma_{\gamma > \beta} ).
\end{equation}
The implication is that, if the above statement is true, then rejecting the transitive closure hypothesis $\gamma > \alpha$ is implied by the two constituent decisions $\gamma > \beta$ and $\beta > \alpha$. 

To show that the statement is in fact true, we use a pigeonhole argument on the relative position of $\beta$: specifically, for any ranking $s \in \Sigma_{\gamma > \alpha}$, it is the case that the ranking is of the form $\underline{X}\gamma\underline{Y}\alpha\underline{Z}$. Thus, we consider the enumerated possibilities $\beta \in \{\underline{X}, \underline{Y}, \underline{Z}\}$. Clearly, $\beta \in \underline{X} \implies \beta > \alpha$ and $\beta \in \underline{Z} \implies \gamma > \beta$. Because $\beta \in \underline{Y}$ implies \emph{both}, we see that any possible ranking in $\Sigma_{\gamma > \alpha}$ implies at least one of $\gamma > \beta$ and $\beta > \alpha$. As such, we see that applying transitive closure does not increase the set of rejected hypotheses in the true ranking space; it only rejects copies of already-rejected belief rankings.

\subsection{A Consistent~\algname~Decision Rule}
We define the~\algname~decision rule as follows: reject candidate hypothesis if any of its belief copies rejects. Because each belief candidate is tied to precisely one pairwise comparison, this is equivalent to rejecting a possible ranking if and only if any pairwise implication is inconsistent with the accumulated evidence. This is exactly the implicit~\algname~decision rule! Thus, there is an~\algname-consistent decision rule in the lifted belief space over full candidate rankings, which when implemented allows for transitive closure without increase in FWER. Thus, applying transitive closure as an~\algname~`wrapper' does not imply rejection of a ranking candidate for which rejection was not already implied by preceding rejections; it only clarifies that, in the compressed space of pairwise comparisons, an additional hypothesis could indeed be rejected as a function of the underlying logic in the lifted space of candidate full rankings. For these reasons, FWER over formal rejection of candidate true rankings \emph{is not increased by the application of transitive closure.} This completes the proof. 

\section{Additional Technical Remarks}
\label{appendix:remarks}

\begin{remark}[IID Assumption]
Note that we do not assume i.i.d increments within policy rollouts, only across different task and environment draws. Without this assumption, policy comparison itself would not be well-defined, as \emph{expected} success rate cannot be meaningfully compared under arbitrary distribution shifts. Separately, one would then have to make assumptions on the type of distribution shift, and there are usually too many possible forms of non-iid structure to characterize effects on FWER in general. Thus, the i.i.d assumption is nearly universally adopted in evaluation both across engineering disciplines and statistical methods (including confidence intervals and Bradley-Terry methods for robot evaluation). In robotics, there exist best practices to more closely realize the assumption~\cite{kress2024robot,barreiros2026careful}, and we briefly note that~\algname's adaptive sampling scheme implicitly integrates some of these practices, including alternating evaluations.

One can use non-iid detection to rigorously check for a broad class of FWER-relevant violations of the i.i.d. assumption, and prioritize detection of particular non-i.i.d. structure like mean drift (arising from camera bumps, mechanical wear, etc). We check the hardware evaluation data for each of the 34 policies considered in the paper; no policies show significant evidence of drift at level $\alpha = 0.1$, indicating the data in the paper does not demonstrate significant violation of the i.i.d assumption or performance degradation, which attenuates concerns of distribution shift in our specific context. 
\end{remark}

\begin{remark}[One-sided vs two-sided tests] 
\label{rem:onesided}
Standard methods in multi-hypothesis correction require one-sided atomic null hypotheses. In our setting, one such hypothesis would be of the form: $\mathcal{H}_0: \mu_i \geq \mu_j$, $\mathcal{H}_1: \mu_i < \mu_j$. Of course, \emph{a priori}, we do not know which policy is better, necessitating the inclusion of the mirrored hypothesis $\mathcal{H}_0: \mu_i \leq \mu_j$, $\mathcal{H}_1: \mu_i > \mu_j$. A potential issue arises with the inclusion of $\mu_i = \mu_j$ in both null hypotheses. This means that, when the means are exactly equal, they belong to the null \emph{in each direction}; thus, \emph{any declarative ranking decision is incorrect} (under the convention that ranking implies non-equality). In this setting, if the probability that either one-sided hypothesis rejects is set to $\alpha$, then the probability of an incorrect decision when the means are equal is bounded at $2\alpha$. As such, a union bound correction (i.e., running each one-sided test at $\alpha/2$) is in general required. If the convention is taken that equality allows either ranking decision to be correct, this is not necessary. Throughout the paper, we present results without this correction, under the reasoning that it is highly unlikely that the complex distribution over tasks and environments would ever induce \emph{exactly identical} means for a policy pair. However, the results can be easily adapted to apply this two-sided correction by instantiating the procedure at significance level $\alpha/2$. 
\end{remark}

\begin{remark}[Evidence Threshold]
    Each $\alpha^{[i]}$ can be understood as controlling the `evidence threshold' for a pairwise comparison: a higher $\alpha^{[i]}$ corresponds to a looser threshold, generally requiring fewer trials before the hypothesis can be rejected. As such, deflating each hypothesis test's significance level by a factor of \(K\) increases requisite hardware evaluations, limiting the statistical power and sample efficiency of the ranking procedure. Thus, our goal is to maximize statistical power by finding looser cutoffs that still guarantee FWER control, thereby deciding as many atomic hypotheses as possible while minimizing the hardware evaluation effort required to do so. 
\end{remark}

\begin{remark}[Statistical Power]
    For single hypothesis testing, $\beta \in [0, 1]$ corresponds to the probability of a failed detection under the alternative. In our multihypothesis setting, we account for partial rankings by assigning $\beta \in [0, 1]$ as the expected fraction of atomic hypotheses which fail to reject (in a given sampling budget), where $(1-\beta)$ is the fraction rejected (statistical power). The `full ranking power' $1-\beta_{FR} \in [0, 1]$ can be defined as the probability that all atomic hypotheses reject (i.e, that we recover a full ranking). It will always be the case that $(1-\beta_{FR}) \leq (1-\beta)$. This can be seen by defining $X \in [0, 1]$ to be the random variable denoting the fraction of rejected hypotheses; then, $1-\beta_{FR} = \mathbb{E}\big[\lfloor X \rfloor \big] \leq \mathbb{E}\big[X\big] = 1-\beta$.
\end{remark}

\section{Choice of Baselines}
\noindent\textbf{FWER-controlling Baselines.} 
We compare against a (sequentialized) standard Bonferroni correction, which controls for the global risk level \(\alpha^*\) by assigning each atomic hypothesis test a portion of the risk budget and controlling the overall error rate via a union bound. This approach was used for rigorous multi-policy comparison in \textbf{LBM} policy evaluations~\cite{barreiros2026careful}. In addition, we also compare against a weighted Bonferroni method, in which atomic hypotheses that are predicted to be difficult to decide based on the prior receive a larger fraction of the significance budget. While Holm~\cite{holm_1979} and fixed sequence~\cite{wiens_fixed_2003} are generally more efficient in the batch regime, they do not admit simple extensions to the sequential setting, unlike Bonferroni. Introducing sequential versions of these tests (as \(\algname_0\) and \(\algname_\infty\), respectively) for the purposes of active evaluation, and proving their FWER control, is the first core contribution of our work. This contribution precisely enables risk recycling for active evaluation and more efficient evaluation decision-making; the comparison of \(\algname_0\) and \(\algname_\infty\) to Bonferroni approaches directly measures this benefit.

Furthermore, \(\algname_0\) and \(\algname_\infty\) also act as an ablation for the benefits of prior information, the second core contribution of this work.
Thus, we quantify these benefits by comparing the fully general \(\algname_\gamma\) to \(\algname_0\) and \(\algname_\infty\); we then quantify the benefits of risk recycling by comparing the latter approaches to sequentialized (weighted) Bonferroni approaches. Per the preceding footnote, we then emphasize that all of these considered approaches already implicitly bake in significant benefits with respect to any and all batch evaluation procedures. 
As a second fair baseline for the benefits of using prior information, we also instantiate a weighted Bonferroni procedure in which the initial risk allocation is in inverse proportion to the predicted difficulty of comparison \(d_i\), to inform understanding of the benefits of graph asymmetry \emph{independent} of the risk recycling mechanism:  
\[
\alpha^{[i]}_0 = \alpha^{*} \frac{\exp(-\gamma d_i)}{\sum_j \exp(-\gamma d_j)}.
\]

\section{Hardware Experiments}
\label{sec:hardware}
\noindent \textbf{RoboArena.} We evaluate seven policies with open-source evaluation datasets from RoboArena~\cite{atreya2025roboarena}. For the prior, we use both heldout evaluations (20 per policy) or world model rollouts in Ctrl-World~\cite{guo2025ctrl} on 40 randomly chosen DROID environments. One of the policies, PG-Bin is an outlier in terms of performance. While the overall correlation between prior policy performance and real progress rate is high with PG-Bin included, the heldout evaluations and world model evaluations do not provide a strong prior ranking on the relative performance of the remaining policies (see~\Cref{fig:roboarena_sim}). 
\begin{figure}[h]
    \centering
    \includegraphics[width=\linewidth]{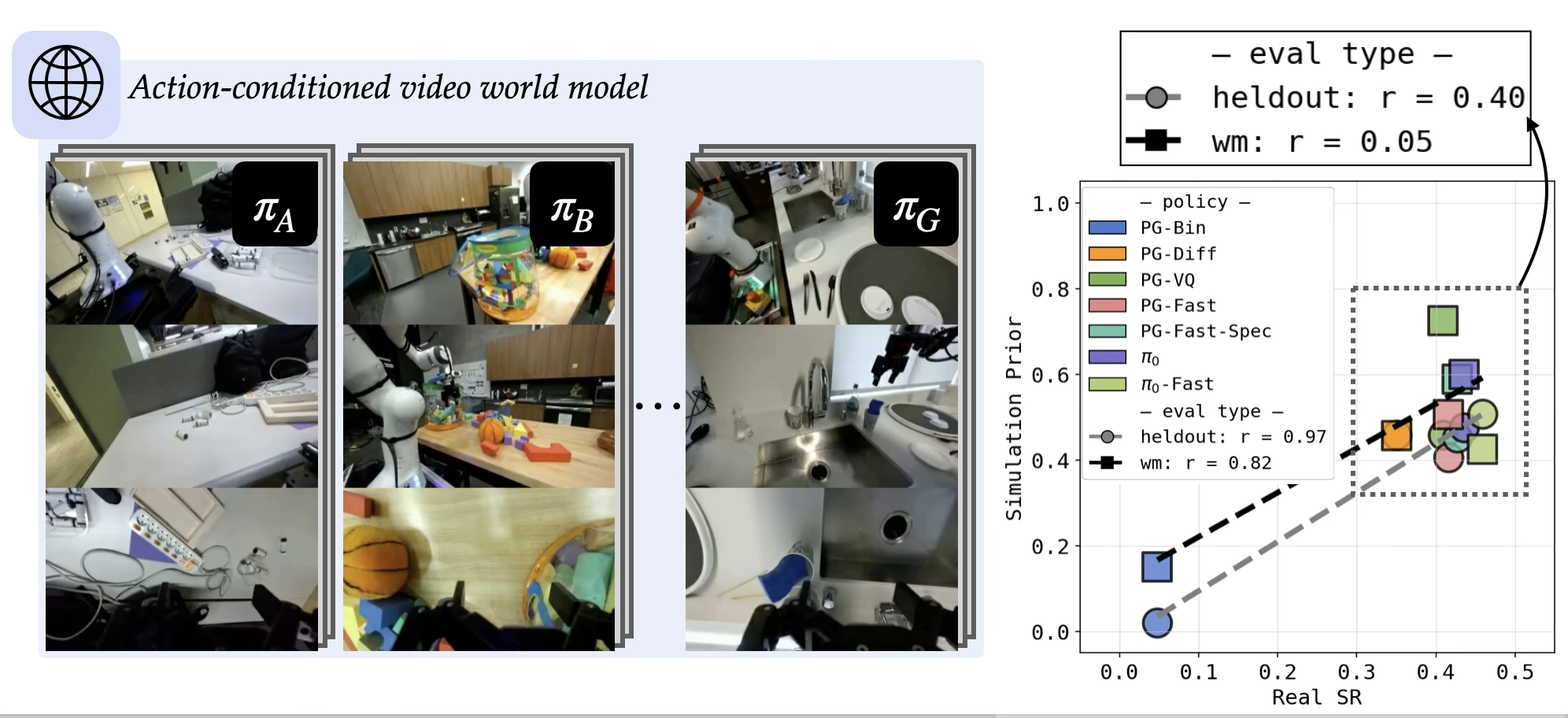}
    \caption{Prior policy performance quality for RoboArena policies using 20 heldout evaluations and 40 random DROID scenes for world model evaluation.}
    \label{fig:roboarena_sim}
\end{figure}

\begin{table}[h]
\centering
\begin{tabular}{clcrrr}
\toprule
Rank & Letter & Policy & BT score & Empirical mean & \# comparisons \\
\midrule
1 & \textbf{J} & Dream-Zero & $1.228 \pm 0.309$ & 0.606 & 58 \\
2 & \textbf{I} & $\pi_{0.5}$ & $0.639 \pm 0.140$ & 0.476 & 564 \\
3 & \textbf{A} & $\pi_0$-Fast & $0.502 \pm 0.133$ & 0.455 & 739 \\
4 & \textbf{K} & FB-195k & $0.428 \pm 0.737$ & 0.479 & 8 \\
5 & \textbf{H} & FB-525k & $0.306 \pm 0.373$ & 0.380 & 33 \\
6 & \textbf{D} & PG-FSQ & $0.283 \pm 0.132$ & 0.418 & 762 \\
7 & \textbf{E} & PG-Diff & $0.262 \pm 0.133$ & 0.404 & 751 \\
8 & \textbf{C} & PG-Fast-Spec & $0.202 \pm 0.126$ & 0.429 & 953 \\
9 & \textbf{B} & PG-Fast & $0.179 \pm 0.128$ & 0.432 & 951 \\
10 & \textbf{N} & DAM & $0.043 \pm 0.296$ & 0.337 & 60 \\
11 & \textbf{L} & FB-420k & $-0.047 \pm 0.394$ & 0.326 & 29 \\
12 & \textbf{F} & $\pi_0$ & $-0.136 \pm 0.128$ & 0.351 & 1009 \\
13 & \textbf{O} & FB-285k & $-0.137 \pm 0.614$ & 0.442 & 12 \\
14 & \textbf{M} & FB-90k & $-0.226 \pm 0.696$ & 0.291 & 10 \\
15 & \textbf{G} & PG-Bin & $-3.525 \pm 0.234$ & 0.047 & 629 \\
\bottomrule
\end{tabular}
\caption{Bradley-Terry-Davidson preference rankings on all 15 RoboArena policies.}
\label{tab:roboarena_bt_rankings}
\end{table}

\textbf{LBM.} Another large-scale evaluation dataset that we consider is the Large Behavior Models dataset from TRI~\cite{barreiros2026careful}. Specifically, we compare three policies in the nominal setting for each task: single task policy, LBM-zeroshot, and LBM-finetuned on the target task. We also compare the performance of these policies in both nominal and out-of-distribution settings. For prior policy performance, we use the simulation evaluations for the corresponding policy and task provided in~\cite{badithela_reliable_2025}. We consider tasks with both simulation and hardware evaluations for these policies such as \emph{PushCoasterToMug}, \emph{TurnMugRightSideUp}, as well as aggregate performance of the policies on these tasks. We have 50 hardware evaluations per policy per task. The real-sim correlations for TRI-LBM experiments  are $0.9965,\, -0.5481,\, 0.6515,\, 0.4117,\, 0.8439$,\, and $0.4226$, respectively. Despite the high correlation in the case of PushCoasterToMug in the nominal environment, two policies (specialist and LBM-zeroshot) are too close to statistically distinguish in 50 trials. Furthermore, the choice of \(\gamma=1\) for \algname may have limited the sample complexity. A more aggressive \(\gamma\) for \(\algname\) could yield better sample complexity. We leave the tuning of parameter \(\gamma\) to future study.

\textbf{DROID-task.}
\begin{figure}[h]
    \centering
    \includegraphics[width=\linewidth]{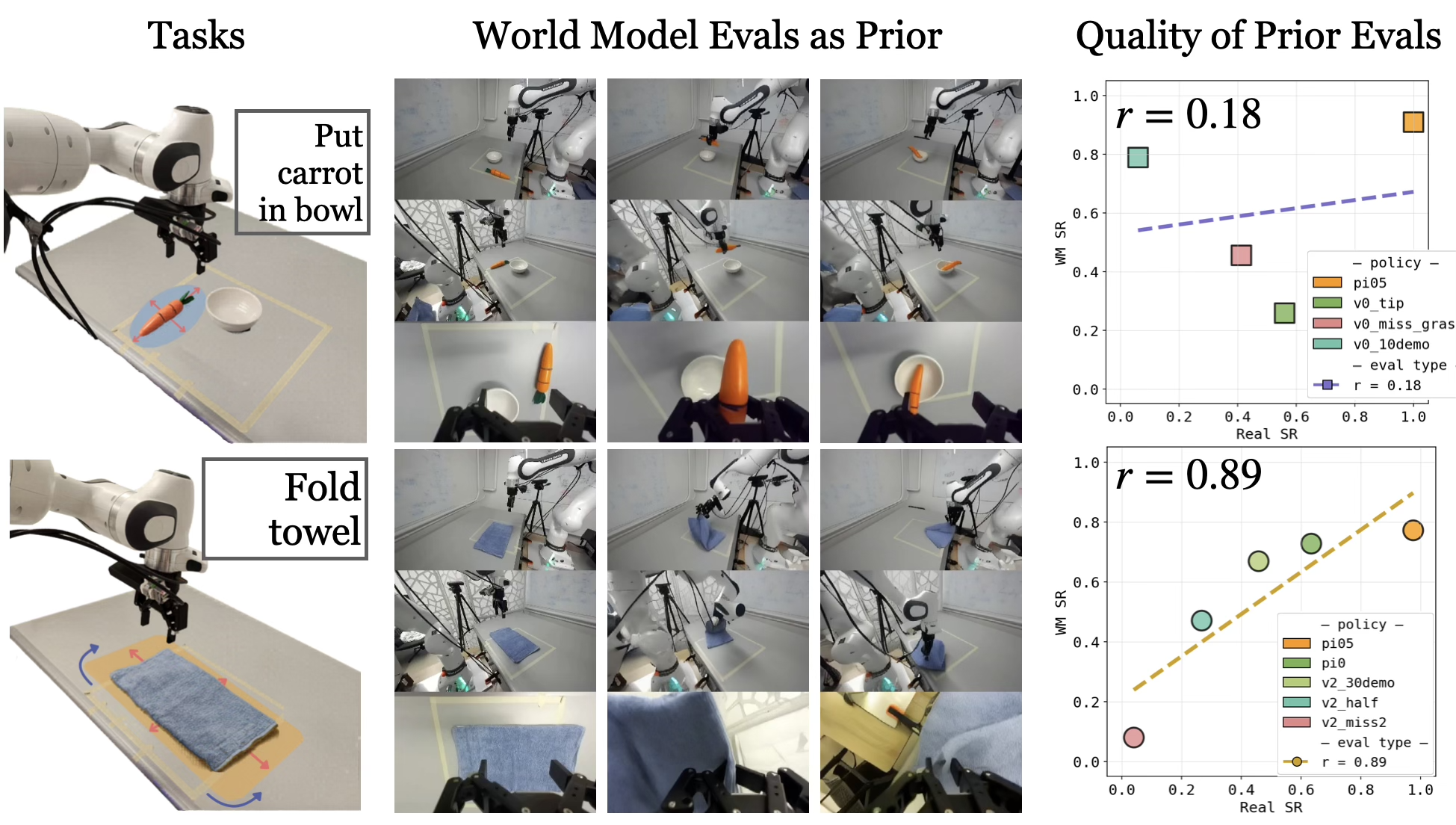}
    \caption{Prior policy performance quality for tasks fold towel and pick up carrot tasks considered in~\textbf{DROID-task}. In the carrot task, we have poor policy ranking information from the prior while for the fold towel task, we have a good policy ranking.}
    \label{fig:droid_local}
\end{figure}
We conduct this set of experiments to illustrate: i) the robustness of~\algname under imperfect priors, and ii) the improvement of ~\algname over using no prior information \(\algname_0\). Both LBM and RoboArena involve evaluation over large and diverse sets of initial conditions, for which existing simulators and video world models yield only weak predictors of real-world performance. While this highlights an important open challenge in constructing accurate evaluation priors, it also provides a challenging testbed for assessing the robustness of \algname when prior information is imperfect.

\noindent \textbf{Policies.} For each task, we trained diffusion policies with varying training datasets to induce a realistic spectrum of policy capabilities. Some policies were trained on only successful human demonstrations, while others were trained with limited data or datasets containing suboptimal and failed trajectories.
In addition to these task-specific diffusion policies, we also evaluate two generalist VLA policies, $\pi_{0}$ and $\pi_{0.5}$~\cite{intelligence2025pi}.

\noindent \textbf{World Model Checkpoints.} For the fold towel task, we use an off-the-shelf video world model, Ctrl-World~\cite{guo2025ctrl}. Ctrl-World is initialized from the stable video diffusion \cite{blattmann2023stable} model and trained on the DROID dataset ~\cite{khazatsky2024droid}. For the carrot task, we report results using a play-data fine-tuned world model variant from \cite{yin2026playworld}, which is further trained on 30 hours of autonomous interaction data and provides better dynamics modeling and visual quality. The initial conditions chosen for evaluation are different from the set of initial conditions used for play data collection.

\subsection{Statistical Reliability of Ranking}
For RoboArena evaluation data, we see instances of both partial and complete rankings: \Cref{tab:roboarena_rankings_full} lists policy rankings for FWER-controlling methods using all available data. As an example of a complete policy ranking, consider a subset of four DROID policies from the RoboArena which have sizeable gaps in empirical performance: \(\pi_0\)-FAST (\A), PG-Diff (\E), \(\pi_0\) (\F), and PG-Bin (\G). For these policies, \algname and its variants recover a complete policy ranking when using world model evaluations as a prior. Note that while \algname and its variants produce a full ranking, weighted Bonferroni is unable to distinguish the top two policies.  This indicates that though weighted Bonferroni allocates a high initial \(\alpha\) for the hypothesis comparing  \(\pi_0\)-FAST (\A) and PG-Diff (\E), it is insufficient to statistically distinguish them.

\textbf{Partial Policy Rankings.} When we consider all seven RoboArena policies included in~\cite{atreya2025roboarena}, we get three additional policies that have very close empirical performance: PG-FAST \((\B)\), PG-FAST-Specialist \((\C)\), and PG-FSQ \((\D)\). As shown in~\Cref{tab:roboarena_rankings_full}, none of the FWER-controlling procedures produce a full ranking as there is not sufficient statistical evidence. For example, the violin plots of $\algname$ (\Cref{fig:active_half_c}) provides a better visual interpretation of the partial ranking: policies \(\A,\, \E,\, \F, \G\) are separable but \(\B,\C,\D\) remain statistically indistinguishable from each other and from policies \(\A\) and \(\E\) as illustrated by their Compact Letter Display (CLD)~\cite{piepho2004algorithm}. In contrast, for the same set of policies, $\algname_\infty$ (fixed sequence approach) is unable to distinguish \(\B,\C,\D\) from policies \(\A, \E\), and \(\F\). Though fixed sequence is the most efficient when all hypotheses can be decided, this is challenging to know \emph{a priori}. This example illustrates the brittleness of fixed sequence to a sub-optimal ordering of hypotheses under imperfect priors.
\begin{table}
\centering
    \resizebox{\linewidth}{!}{%
        \begin{tabular}{c|l|c|c}
            \hline
            & \textbf{Methods}
            & RoboArena4
            & RoboArena7
            \\ 
            \hline
            \multirow{3}{*}{\textbf{Ours}}
            & \algname
            & {\(\A > \E > \F > \G\)}
            & \(\A > \E > \F > \G\), \(\partialtxt{\{\A\B\C\D\}}\), \(\partialtxt{\{\B \C \D \E\}}\) \\
            & $\algname_{0}$
            & {\(\A > \E > \F > \G\)}
            & \(\A > \E > \F > \G\), \(\partialtxt{\{\A\B\C\D\}}\), \(\partialtxt{\{\B \C \D \E\}}\) \\
            & $\algname_{\infty}$
            & {\(\A > \E > \F > \G\)}
            & \(\A > \E > \F > \G\), \(\partialtxt{\{\A \B \C \D\}}\), \(\partialtxt{\{\B\C \D \E\}}\), \(\partialtxt{\{\B \C \D \F\}}\) \\
            \hline
            \multirow{2}{*}{\textbf{Baselines}}
            & Bonferroni
            & {\(\A > \E > \F > \G\)}
            & \(\A > \E > \F > \G\), \partialtxt{\(\{\A \B \C \D\}\)}, \partialtxt{\(\{\B \C \D \E\}\)}\\
            & Weighted Bonferroni
            & {\(\partialtxt{\{\A \E\}} > \F > \G\)}
            & {\(\partialtxt{\{\A \B \C \D \E\}} > \F > \G\)}\\
            \hline
        \end{tabular}%
    }
    \caption{%
        Ranking DROID policies from RoboArena~\cite{atreya2025roboarena} with all available evaluations using world model evaluations as a prior. Each capitalized alphabet represents a policy: \(\pi_0\)-FAST (\A), PG-FAST \((\B)\), PG-FAST-Specialist \((\C)\), PG-FSQ \((\D)\), PG-Diff (\E), \(\pi_0\) (\F), and PG-Bin (\G). Policies in a group and highlighted in \partialtxt{gold} are not statistically distinguishable by the method.
    }
    \label{tab:roboarena_rankings_full}
\end{table}

\subsection{Sample Efficiency}
We include additional results (\Cref{fig:tri_nominal}) from the TRI dataset, which continues to show the advantages of \algname over other baselines.
\begin{figure}[!htb]
    \centering

    \begin{subfigure}[t]{0.3\linewidth}
        \centering
        \includegraphics[width=\linewidth]{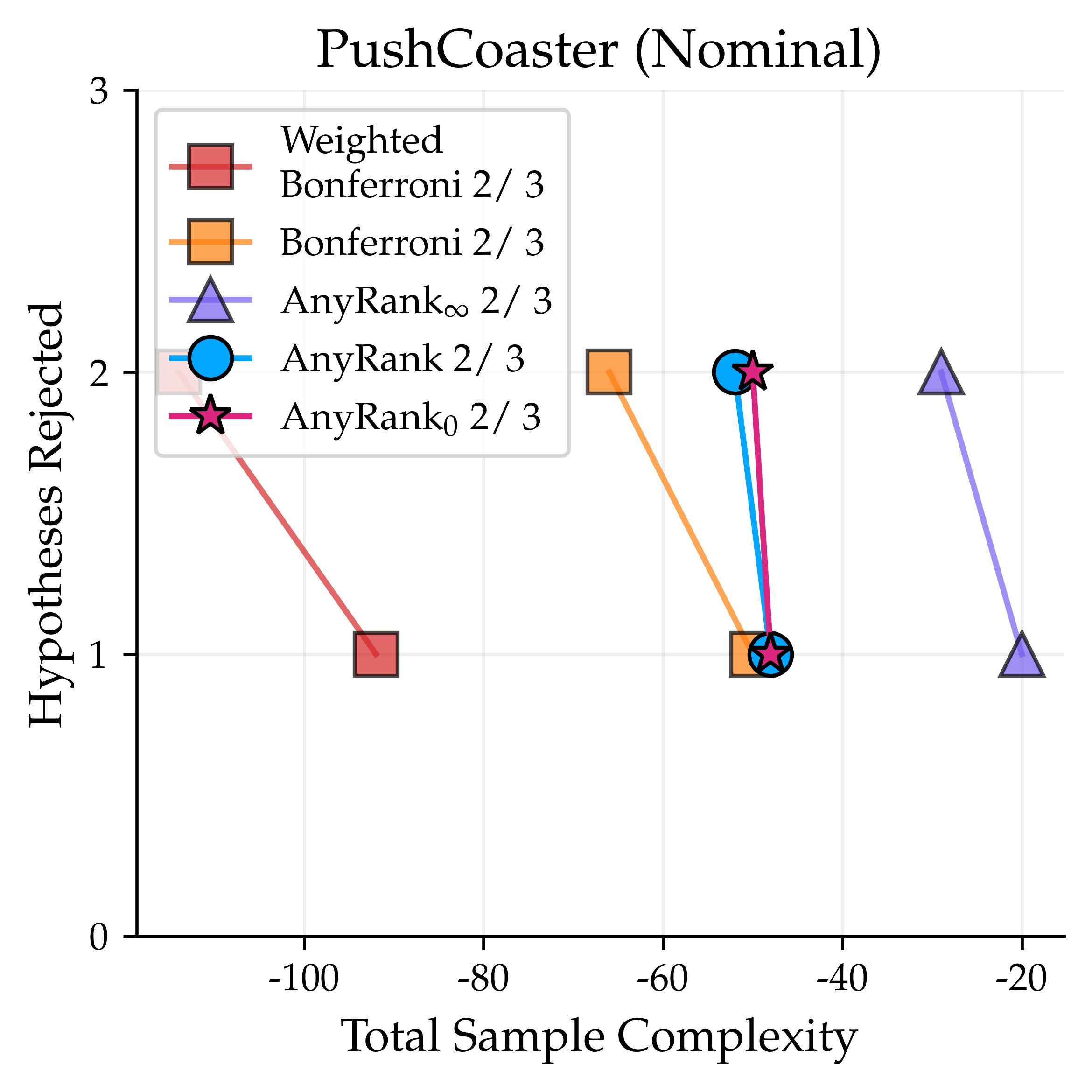}
        \caption{PushCoaster.}
    \end{subfigure}
    \hfill
    \begin{subfigure}[t]{0.3\linewidth}
        \centering
        \includegraphics[width=\linewidth]{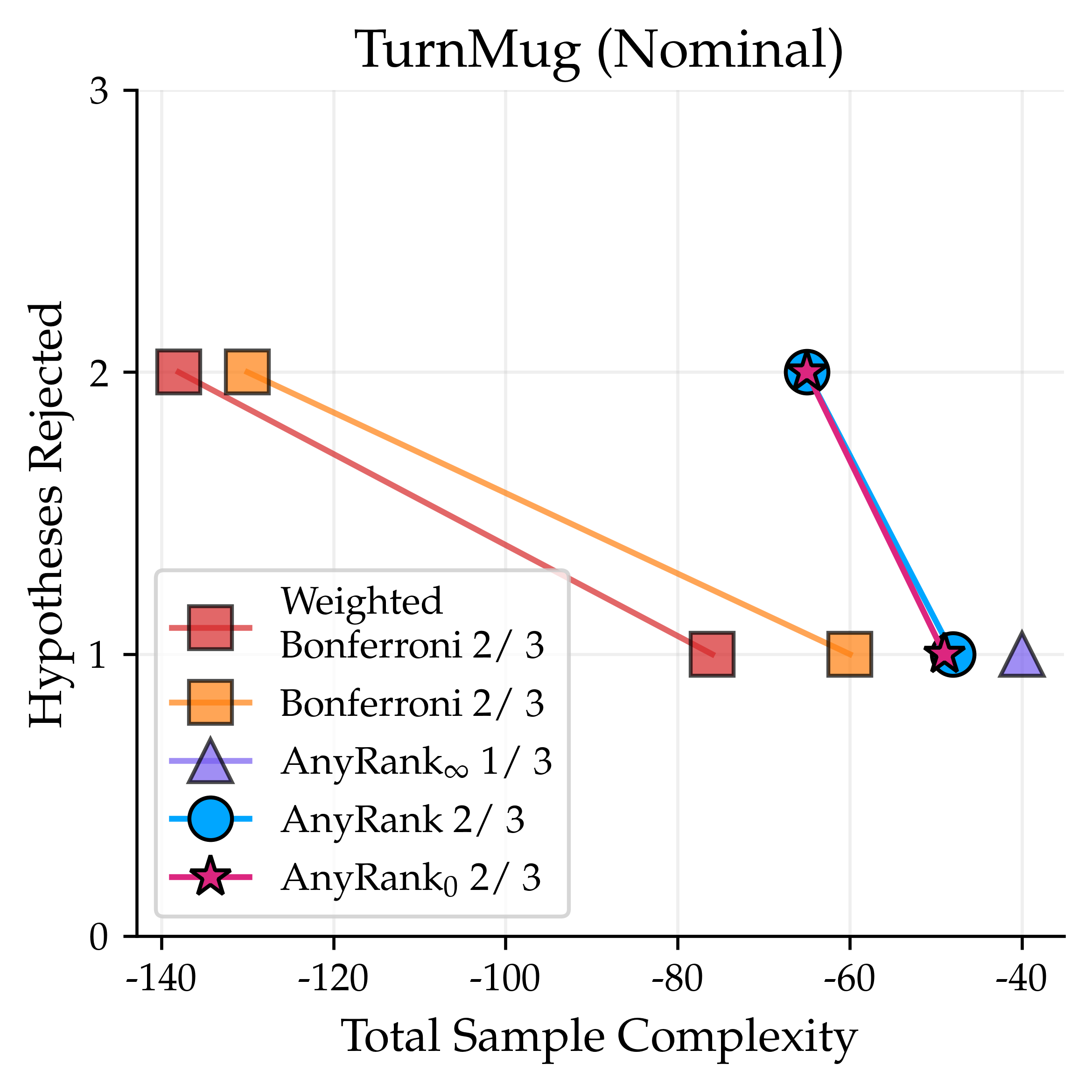}
        \caption{TurnMug.}
    \end{subfigure}
    \hfill
    \begin{subfigure}[t]{0.3\linewidth}
        \centering
        \includegraphics[width=\linewidth]{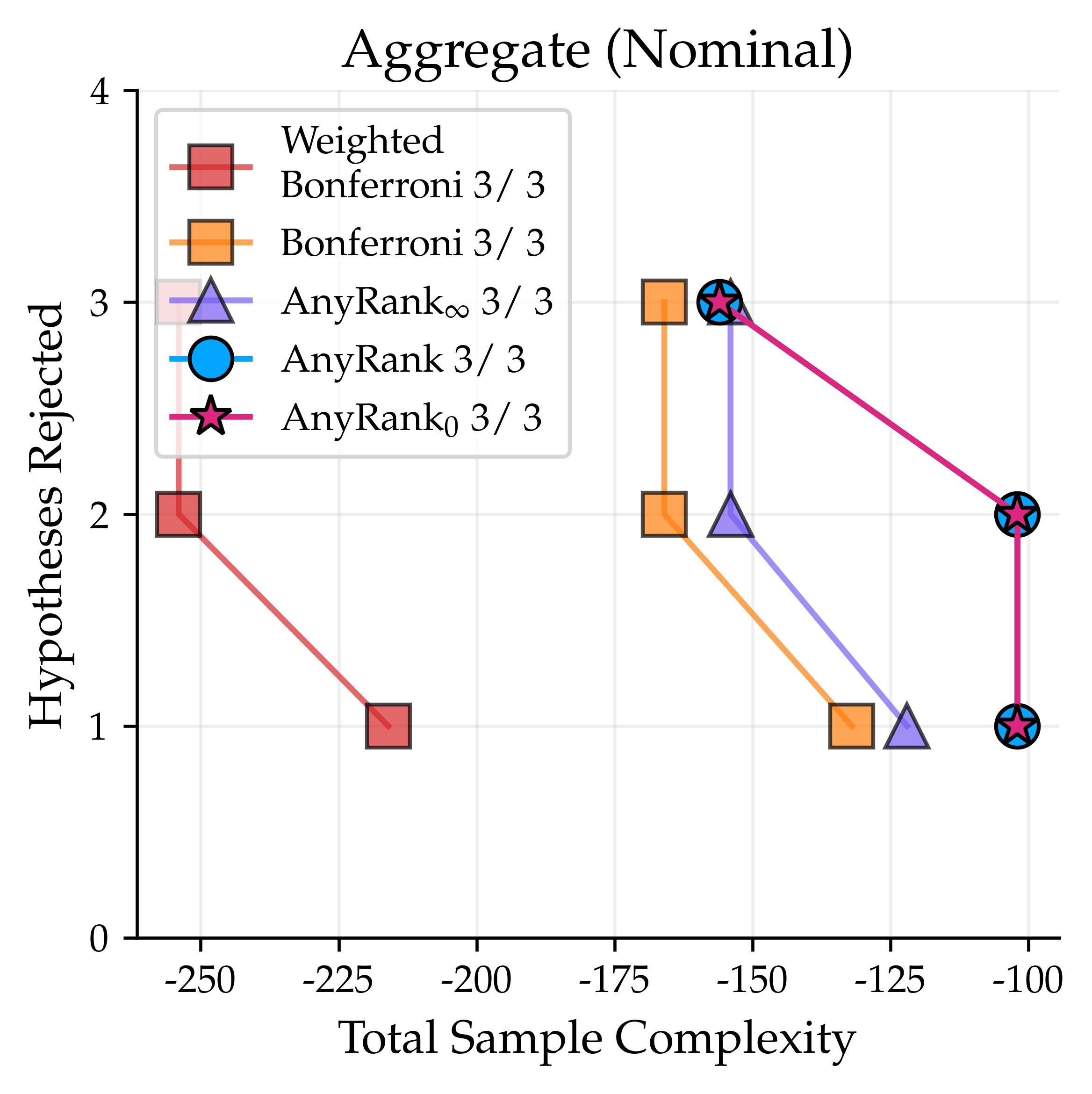}
        \caption{Aggregate.}
    \end{subfigure}
    \caption{\textbf{Pareto Curves for TRI-LBM under nominal evaluation environments.} }
    \label{fig:tri_nominal}
\end{figure}

\subsection{Reliability Against Imperfect Simulation Priors} 
\begin{figure}[h]
    \centering
    \includegraphics[width=\linewidth]{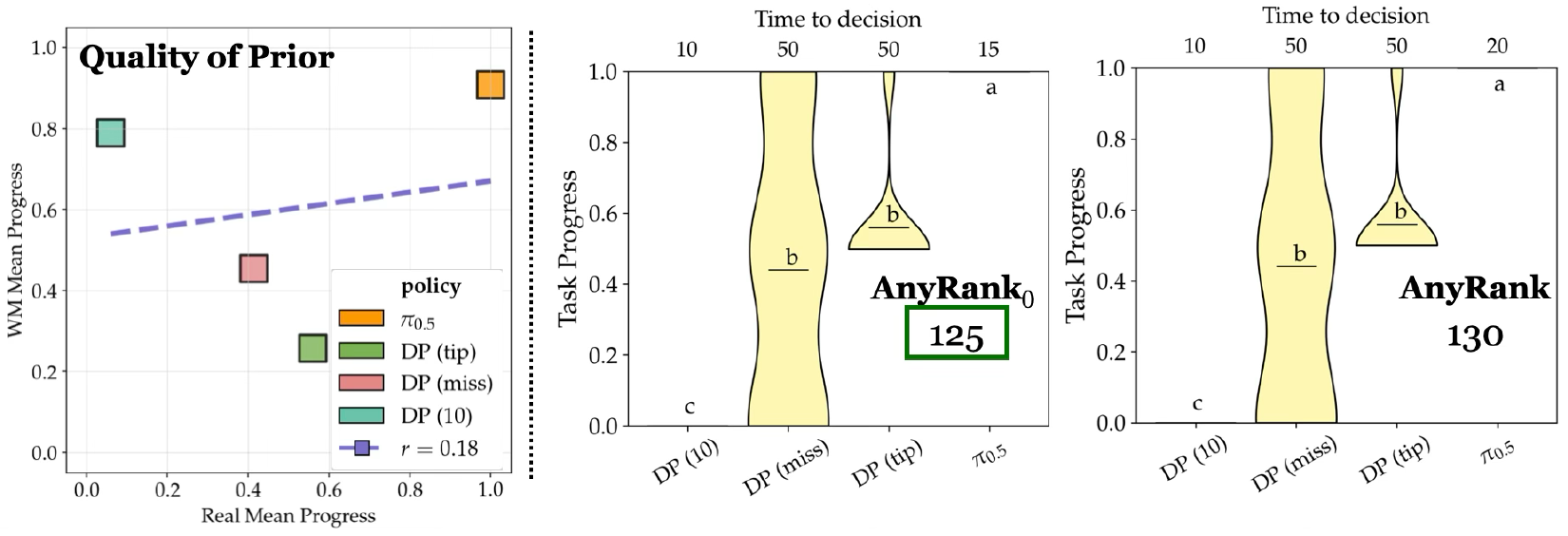}
    \caption{
     \textbf{Reliability against imperfect priors.} \algname (at \(\gamma=10\)) requires slightly more trials (130) than~$\algname_{0}$, which does not use any prior information.
    }
    \label{fig:bad_prior}
\end{figure}
To illustrate the robustness of ~\algname to imperfect prior, we consider the \textbf{DROID-task} series of experiments for the \emph{Put carrot in bowl} task. The diffusion policies and generalist policies compared are illustrated in~\Cref{fig:droid_local}. Note that in this instance, the correlation between policy performance in the world model and the real world is low: \(r=0.18\), indicating that the prior information is a poor proxy of real-world policy performance. In \Cref{fig:bad_prior}, observe that with very weak and mismatched priors, \algname can become marginally less efficient than $\algname_0$, but does not sacrifice validity. The \textbf{RoboArena} and \textbf{LBM} experiments also correspond to settings with weak simulation priors, which results in \algname being inefficient compared to $\algname_{0}$ that does not use the prior information in constructing the graph weights.

\section{Synthetic Experiments} 
\label{subsec:BT_vs_Holm}
On real-world evaluations, the \emph{true mean} of task progress of a policy is unknown. Given this lack of access to the true labels, one cannot assess \textbf{FWER Control} or \textbf{Asymptotic Inconsistency}. For real data, we can partially validate \textbf{Miscalibration} of baseline procedures via comparison of atomic hypotheses to~\algname, which is guaranteed to be correct with probability $1-\alpha$. Additionally, we can assess \textbf{Inconsistency} by running procedures on data subsets of varying sizes, both in the number of trials as well as in the number of policies. To validate the remaining properties, we require the construction of synthetic data for which oracle knowledge of the true underlying means is known, in order to validate using the true labels over \emph{many independent evaluation datasets} that \textbf{FWER Control} and \textbf{Asymptotic Consistency} are jointly maintained in general. 


\subsection{Ranking Procedures}
\textbf{Bradley-Terry Variants.}
As in \Cref{sec:evaluation}, in these experiments we use a standard Bradley-Terry procedure (``BT (Classical)" in \Cref{tab:roboarena_rankings}), as well as the variant constructed in \cite{atreya2025roboarena} (``BT (EM)"). These methods are hypothesized to not control FWER, nor to be consistent estimators for general problems of ranking mean policy performance.

\textbf{Bootstrapped Bradley-Terry Approaches.}
In order to (asymptotically) ensure FWER control in the BT procedures, we implement a bootstrapped BT for each batch. To bootstrap, we begin with a BT variant (described in the preceding subsection) and generate $M=5000$ redraws i.i.d. with replacement from the batch dataset, with each redraw of size equal to the nominal batch dataset, $N$, and take the empirical $\{\alpha/(2K), 1-\alpha/(2K)\}$ quantiles of the BT scores for each policy independently. We then compare the policies pairwise by utilizing these confidence intervals; any pair of intervals with \emph{no overlap} correspond to a significant rejection. A Bonferroni correction over these pairwise comparisons guarantees (asymptotic) $1-\alpha$ FWER control. Unfortunately, because these intervals are constructed in a brittle way, attempting to utilize graphical methods or other more efficient multitest corrections is infeasible, because it would require computing a combinatorial number of different confidence intervals.

In this setting, we hypothesize that for most problems, the rankings will be FWER-robust in that decisions will only be made if there is significant evidence. However, we hypothesize that it will result in a significant loss of statistical power. Additionally, we hypothesize that this variant will inherit the inconsistency properties from the standard BT variants. 


\begin{figure}[htb]
    \centering
    \includegraphics[width=0.9\linewidth]{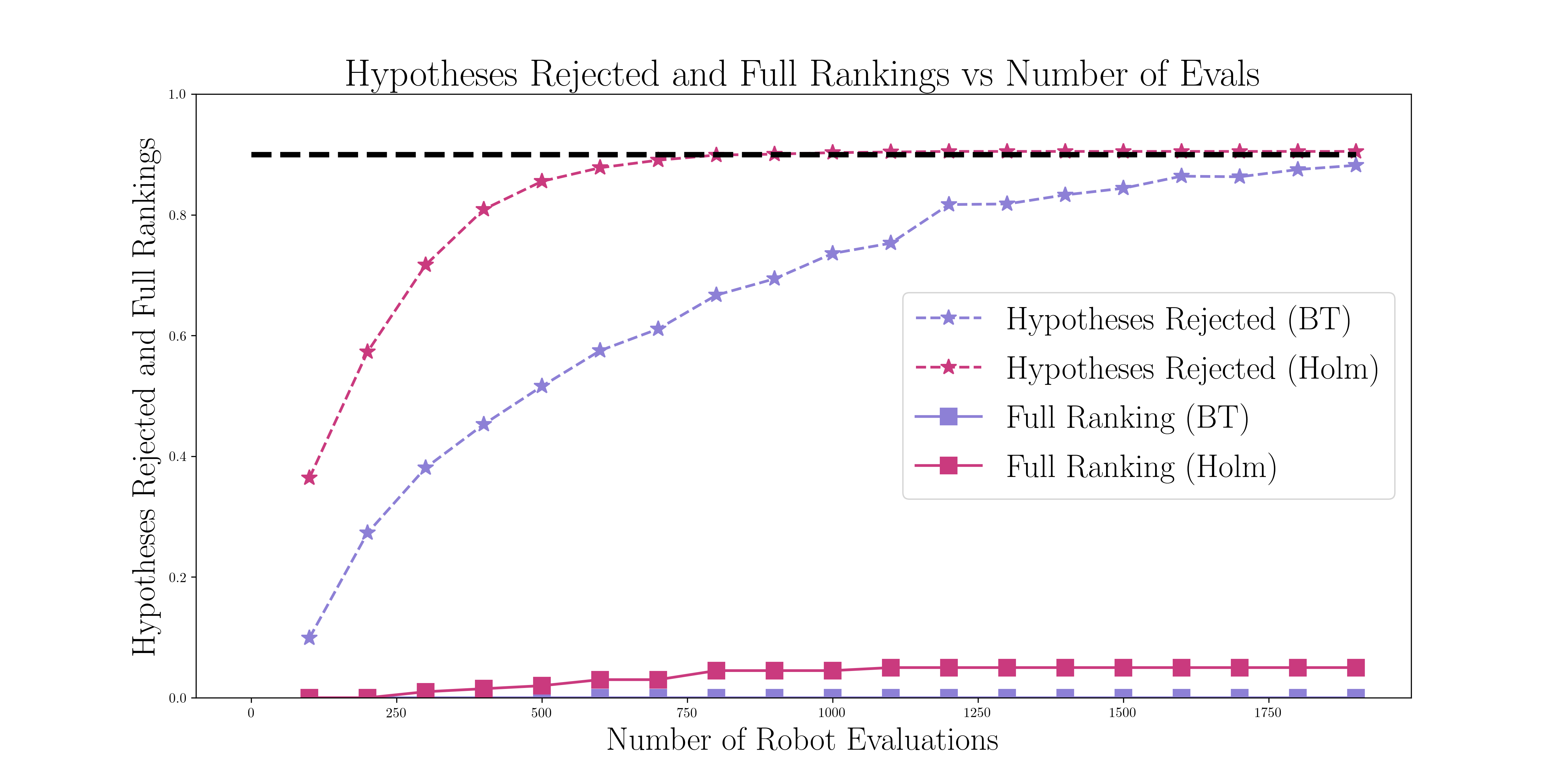}
    \caption{\textbf{Validating FWER Control with Synthetic Data.} Setting \textbf{HO-I}: we demonstrate the improved efficiency of~$\algname_{0}$~over BT$_\alpha$ on synthetic data. Here, there are five policies corresponding to 10 pairwise comparisons. Two policies share the same mean, resulting in an \emph{infeasible} full ranking. We demonstrate here the time to a complete partial ranking (denoted by dashed lines) for each approach, as well as the time to full rankings (denoted by solid lines). \emph{Note} that a full ranking in this context is equivalent to a Type-1 error, as one of the atomic comparisons [between the policies with equal means] should never decide. We observe that~$\algname_{0}$~is able to better utilize the risk budget $\alpha = 0.05$ and active sampling in order to more quickly arrive at complete partial rankings.}
    \label{fig:alpha_control_synthetic}
\end{figure}
\subsection{Experiments with Synthetic Data}
The central limitation of Bradley-Terry procedures is the difficulty of ascertaining a level of Type-1 Error on practical datasets. For instances in which the underlying metric of interest is unobservable, BT often remains the best option. However, in the robotics evaluation context, BT is often used to map evaluator preference to rankings of overall performance; this setting is therefore fundamentally different, because robot performance is directly observed in the evaluation rubric.\footnote{\label{footnote:analogy}BT methods are generally used in ranking contexts, like athletic performance, when an absolute and shared performance measure is \textbf{not} present. For example, this is the case in sports driven by win-loss records -- soccer, baseball, basketball -- where the observed quantities (wins and losses) are only proxies for the quantity of interest (overall ability, quality of play, etc). This mismatch necessitates BT-type rectification. By contrast, BT is less effective for sports with shared performance measures, like track and field events. Here, the win-loss measure is derived from a shared, global measure; the race times themselves are directly observed. In using BT in ChatbotArena \cite{chiang2024chatbot} for LLM-based evaluation, there is value added because the measure of LLM response quality is imprecise, as many responses equally satisfy the query. In the robotics context, preferences may be a useful proxy for some estimation problems, but do not make sense as a proxy for empirical performance \emph{when empirical performance is directly observed.}} As a result, properties like FWER control become far more important, as there are many plausible methods to construct a ranking. 

To validate Type-1 Error (FWER) control, we run \algname, Bradley-Terry (BT), and Bootstrapped Bradley-Terry (BT$_\alpha$) on many instances of synthetic Bernoulli data. This data structure corresponds to the setting of ranking policies by task success rate. By running many independent instances of ranking problems, we can validate the properties of the ranking algorithm (which in general cannot be done on real datasets). 

For each synthetic policy ranking instance, the maximal number of (fictitious) evaluations per policy is set as $N_{max} = 2000$. There are $M = 5$ policies, each with associated true (known) Bernoulli mean $\mu_i$. For all BT variants, we add a small amount of random, continuous noise to easily break ties (ensuring that the resulting data variance is strictly less than that of the Bernoulli data used by \algname), and assign preferences based on pairwise scores (higher evaluation mean implies the method is preferred). The global ranking is FWER controlled at level $\alpha = 0.05$. 

There are three semantically different evaluation settings to consider, categorized by the difficulty of the ranking components (homogeneous or heterogeneous) and the feasibility of a full ranking (feasible or infeasible):  
\begin{itemize}
    \item \textbf{HO-F}: ``Homogeneous Feasible;" $\{\mu_i\} = \{0.10, 0.30, 0.50, 0.70. 0.90\}$
    \item \textbf{HE-F}: ``Heterogeneous Feasible;" $\{\mu_i\} = \{0.20, 0.35, 0.45, 0.65, 0.90\}$
    \item \textbf{HO-I}: ``Homogeneous Infeasible;" $\{\mu_i\} = \{0.20, 0.40, 0.40, 0.60, 0.80\}$
\end{itemize}
For each setting, we rerun the entire ranking procedure on 200 independent redraws to validate FWER control and test empirical power. As previously mentioned, a family-wise error is synonymous with an incorrect ranking, and we seek to validate that the error rate is controlled at level $\alpha^*$. Power is measured by a sharp time-to-full-ranking metric, and a softened version of the metric corresponding to the fraction of hypotheses rejected vs time (if the fraction equals 1, then we have a full ranking).

\subsubsection{Limitations of Bradley--Terry: Asymptotic Inconsistency}
\label{subsubsec:experiments_asymptotic_inconsistency}
On synthetic data, we illustrate a separate and important phenomenon of applying BT procedures for ranking: the estimated Bradley-Terry ranking may converge to an incorrect ordering, even with infinite data and when directly using empirical performance comparison as the preference rule. To illustrate this phenomenon, we construct a hypothetical three-policy ranking problem, where the task metric is discrete partial credit over ten subgoals, resulting in eleven possible evaluation outcomes for each rollout ($\{0, 0.1, 0.2, ..., 1.0\}$) on a given task instance. We consider three policies, two of which display high variance, achieving rewards of Bernoulli($0.3$) and Bernoulli($0.7$), respectively. The last policy is low-variance, achieving scores of $0.8$ and $0.9$ with equal probability. We have defined by construction a setting where $\mu_1 = 0.3$, $\mu_2 = 0.7$, and $\mu_3 = 0.85$.
\begin{figure}[t]
    \centering
    \begin{subfigure}[b]{0.48\linewidth}
        \centering
        \includegraphics[width=\textwidth]{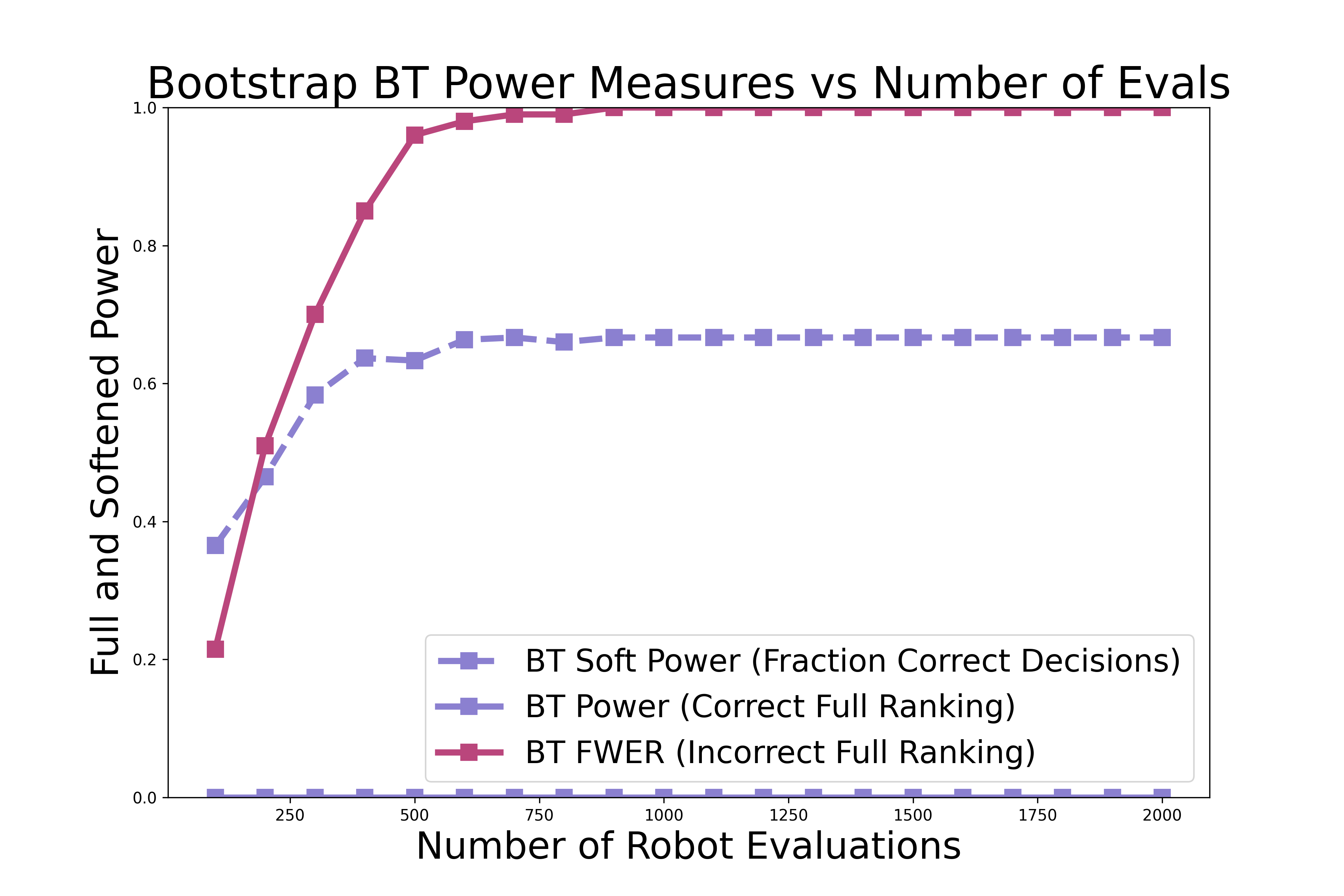}
        \caption{Here, we show the FWER (erroneous ranking fraction) as a function of dataset size. In this setting, BT$_\alpha$ is increasingly confident of a ranking that is inconsistent with empirical policy performance, causing the FWER to grow larger as the dataset size increases.}
        \label{fig:inconsistent_bt}
    \end{subfigure}
    \hfill 
    \begin{subfigure}[b]{0.48\linewidth}
        \centering
        \includegraphics[width=\textwidth]{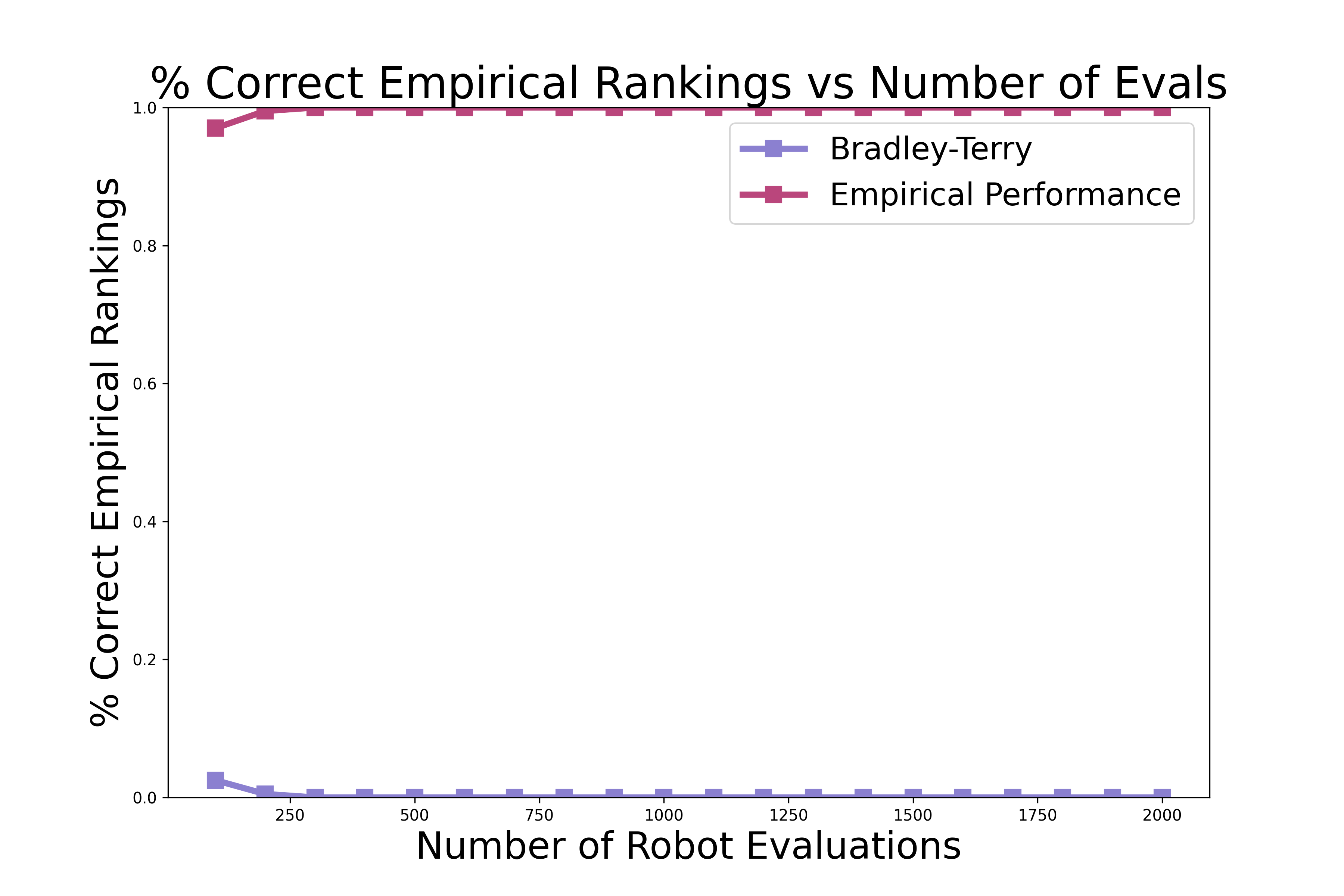}
        \caption{The counterexample constructed to illustrate \Cref{fig:inconsistent_bt} does not affect direct policy performance-based evaluation approaches, as the empirical orderings remain consistent. Procedures like~\algname~are thus preferable for policy ranking.}
        \label{fig:right}
    \end{subfigure}
    \caption{\textbf{Demonstrating Bootstrap BT Inconsistency.} We demonstrate by counterexample (constructed in \Cref{subsubsec:experiments_asymptotic_inconsistency}) the inconsistency of Bradley-Terry approaches in general, even in the presence of robust bootstrapping corrections.}
    \label{fig:combined}
\end{figure}


 We now show that BT methods are \emph{inconsistent} in this case: that is, in the infinite-data limit, any BT method using empirical performance as the preference indicator will rank the policies inconsistently with the actual order of said policies' means. Concretely, the true order is $\mu_1 < \mu_2 < \mu_3$; any Bradley-Terry approach will instead return $\mu_1 < \mu_3 < \mu_2$. The intuition is that progress-based preference can reward variance, in this instance by allowing the Ber($0.7$) policy to `beat' both of the other policies approximately 70\% of the time. This manifests in the behavior illustrated in \Cref{fig:inconsistent_bt}, where FWER (measured by erroneous rankings) is not controlled, even when applying a bootstrap correction at level $\alpha$ (denoted BT$_\alpha$). This occurs precisely because more data causes the Bradley-Terry procedure to become more confident in the incorrect ranking placing $\mu_2 = \mathbb{E}[\text{Ber}(0.7)]$ ahead of $\mu_3 = 0.85$.
 
\subsection{Validating Type-1 Error (FWER) Control on Synthetic Data}
We begin with case \textbf{HO-I}, in which two policies have the same mean and therefore cannot be ordered (under the convention that ordering $\mu_i$ vs $\mu_j$ implies that $\mu_i \neq \mu_j$). In this setting, nine of the ten hypotheses are decidable (resulting in a partial ranking). In \Cref{fig:alpha_control_synthetic} we illustrate the FWER control of both \algname and BT$_{\alpha}$ (i.e., bootstrapped BT) at level $\alpha^* = 0.05$ (solid lines). In this setting, a full ranking corresponds to an error, as it implies we have ordered the two policies that have the same mean; therefore, any instance in which that occurs is a Type-1, or Family-Wise, error. Consequently, the solid lines correspond to the Type-1 Error rate. As can be observed, our method uses nearly the full risk budget (an empirical error rate of $0.045$), while BT$_\alpha$ is safe, but much more conservative. 

By contrast, as might be expected, the standard BT approach (not shown in \Cref{fig:alpha_control_synthetic}) results in a coin flip: in the most generous case, the ranking is correct $50\%$ of the time.\footnote{This is because we allow BT to count one of the policies as $\epsilon$-better than the other. In fact, in the theoretical case where the means are identical, BT is wrong essentially 100\% of the time, because it almost never produces ties. Thus, any ordering is technically incorrect.} Importantly, the user will have no way to quantitatively evaluate this instability on real data.

\begin{figure}[htb]
    \centering
    \includegraphics[width=\linewidth]{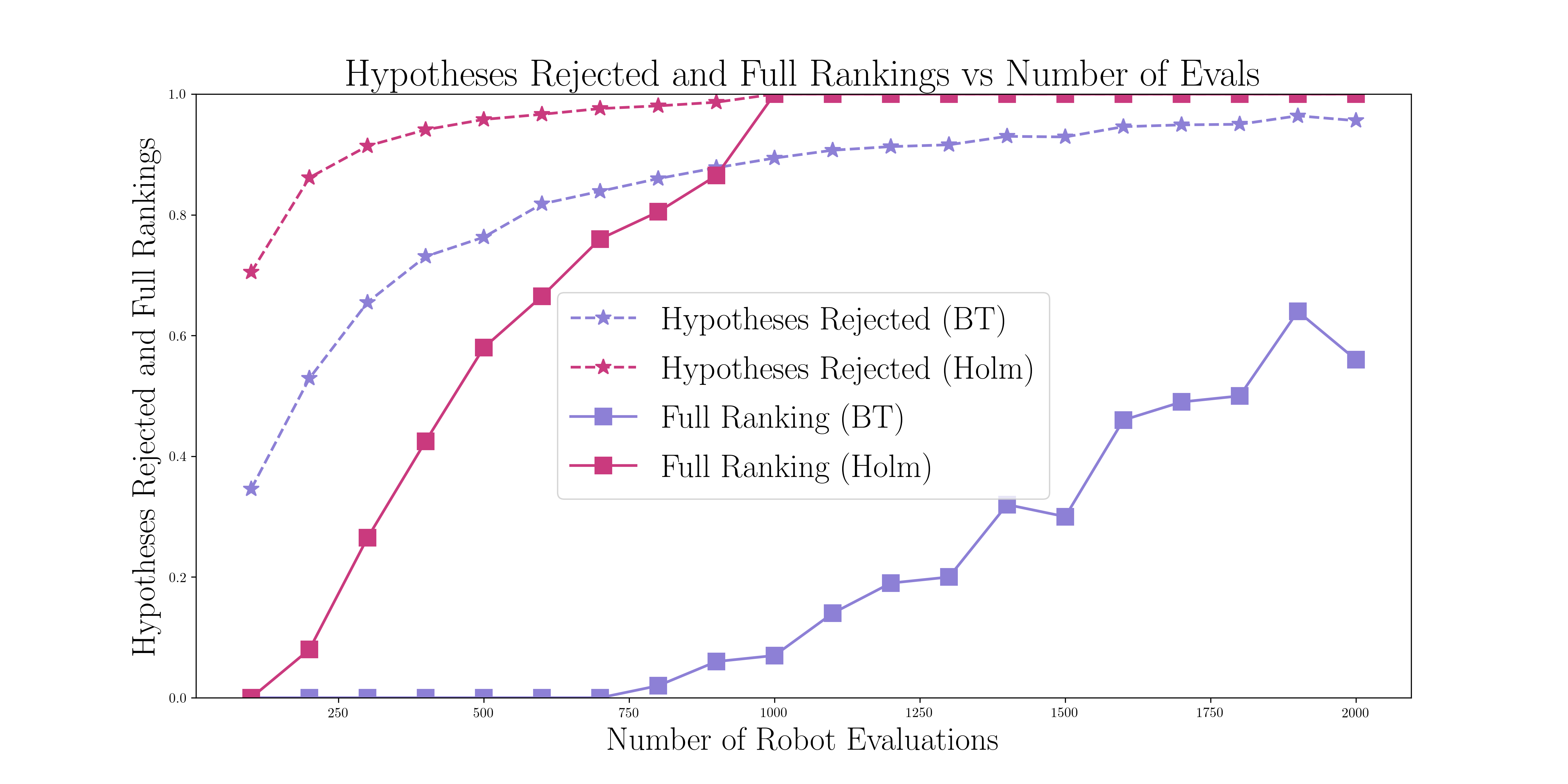}
    \caption{\textbf{Testing Statistical Power with Synthetic Data.} Setting \textbf{HE-F}: we demonstrate the improved efficiency of~$\algname_{0}$~over BT$_\alpha$ on synthetic data. Here, there are five policies corresponding to 10 pairwise comparisons. The policies have unequally spaced means, resulting in heterogeneity in time to \emph{feasible} full ranking. We demonstrate here the time to $k$-hypothesis partial rankings (dashed lines) for each approach, as well as the empirical time to \emph{full ranking} (solid lines). We observe that~$\algname_{0}$~is able to better utilize the risk budget $\alpha = 0.05$ and active sampling in order to more quickly arrive at partial and complete rankings across the spectrum of sample complexity, illustrated by the blue lines being above their red counterparts.}
    \label{fig:power_heterogeneous_synthetic}
\end{figure}
\subsection{Note: Empirical Performance is as Strong as Non-Bootstrapped BT}
We briefly develop the analogy given in \Cref{footnote:analogy} by considering the efficacy of preference-based BT (using preferences based on empirical performance) vis-a-vis a related non-FWER controlling approach: ranking by empirical performance. In \Cref{fig:invalid_approaches_FRF_H}, \Cref{fig:invalid_approaches_FRI}, and \Cref{fig:invalid_approaches-FRF_I} we show that across data scales in the synthetic setting, using the current empirical means to rank the policies is as effective on aggregate as preference-based BT rankings, when FWER control is not a concern. This result reinforces the notion that, when the ultimate metric of interest is directly observable, it is generally optimal to estimate that metric based on the direct constituent observations of the metric, and not on proxy measures.\footnote{Assuming, for example, that the relative data availability is similar, etc. \citet{badithela_reliable_2025} is an example where the incorporation of a cheaper simulation proxy saves evaluation on more expensive hardware. In our setting, by contrast, two policy rollouts yield either a preference score or two evaluation rollout scores at approximately equal cost.} 
\subsection{Investigating Empirical Power Of Valid Methods On Synthetic Data}
To investigate empirical power, we proceed in reverse order through the three cases. Beginning with \textbf{HO-I} in \Cref{fig:alpha_control_synthetic}, we now consider the progress towards a completed partial ranking (the limit for this case) for both bootstrapped BT and \algname, as shown in the dashed lines. Consistent with conservatism interpretation in the preceding section, we observe that \algname is significantly quicker to reject more hypotheses (illustrated by the blue curve being strictly higher than the red curve), in particular as the hypotheses become more difficult (this is illustrated by the horizontal gap growing as the fraction of rejections increases). 

\begin{figure}
    \centering
    \includegraphics[width=0.9\linewidth]{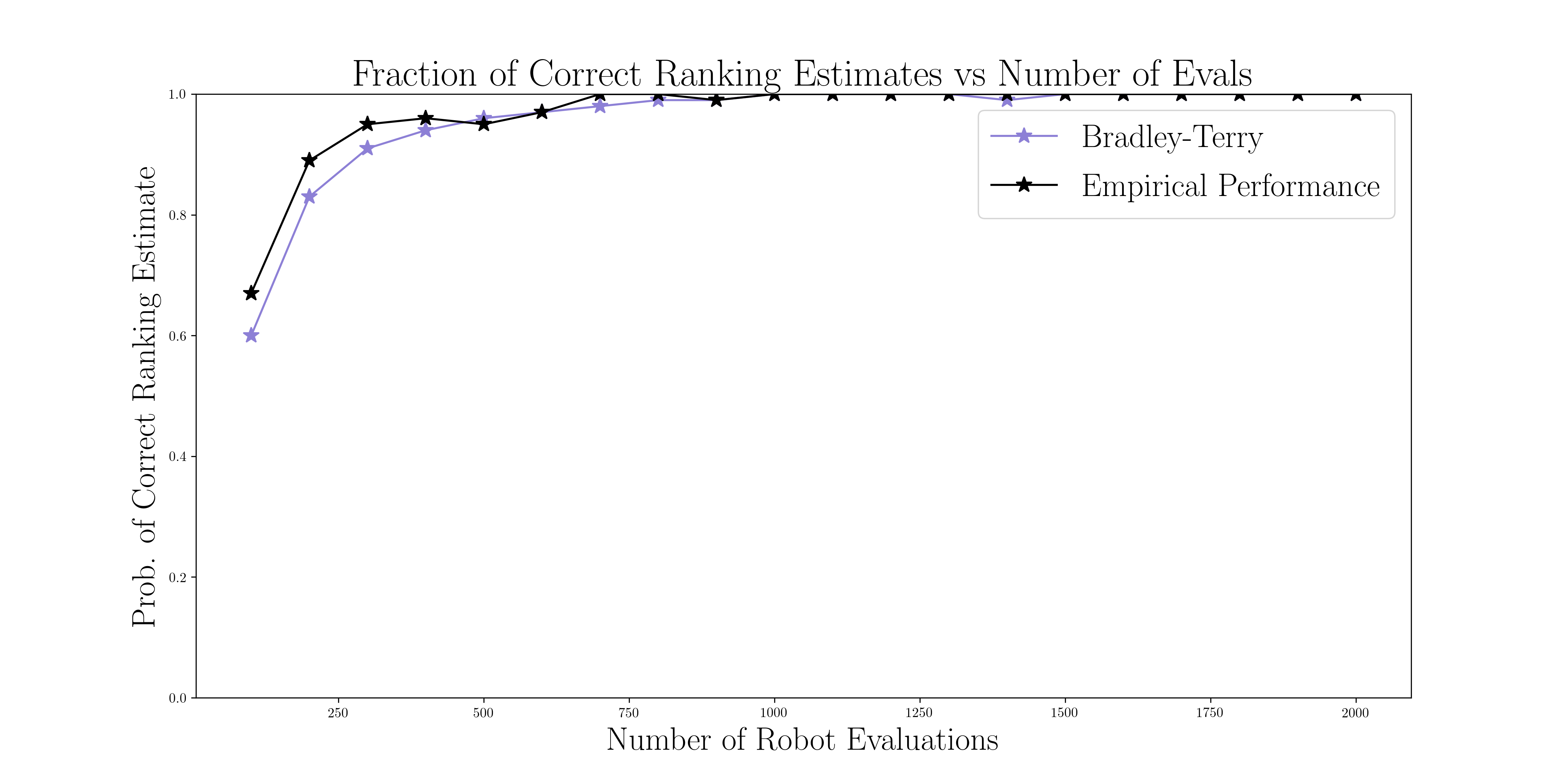}
    \caption{\textbf{Empirical Performance is As Good as Preferences (Feasible Homogeneous).}}
    \label{fig:invalid_approaches_FRF_H}
\end{figure}

\begin{figure}
    \centering
    \includegraphics[width=0.9\linewidth]{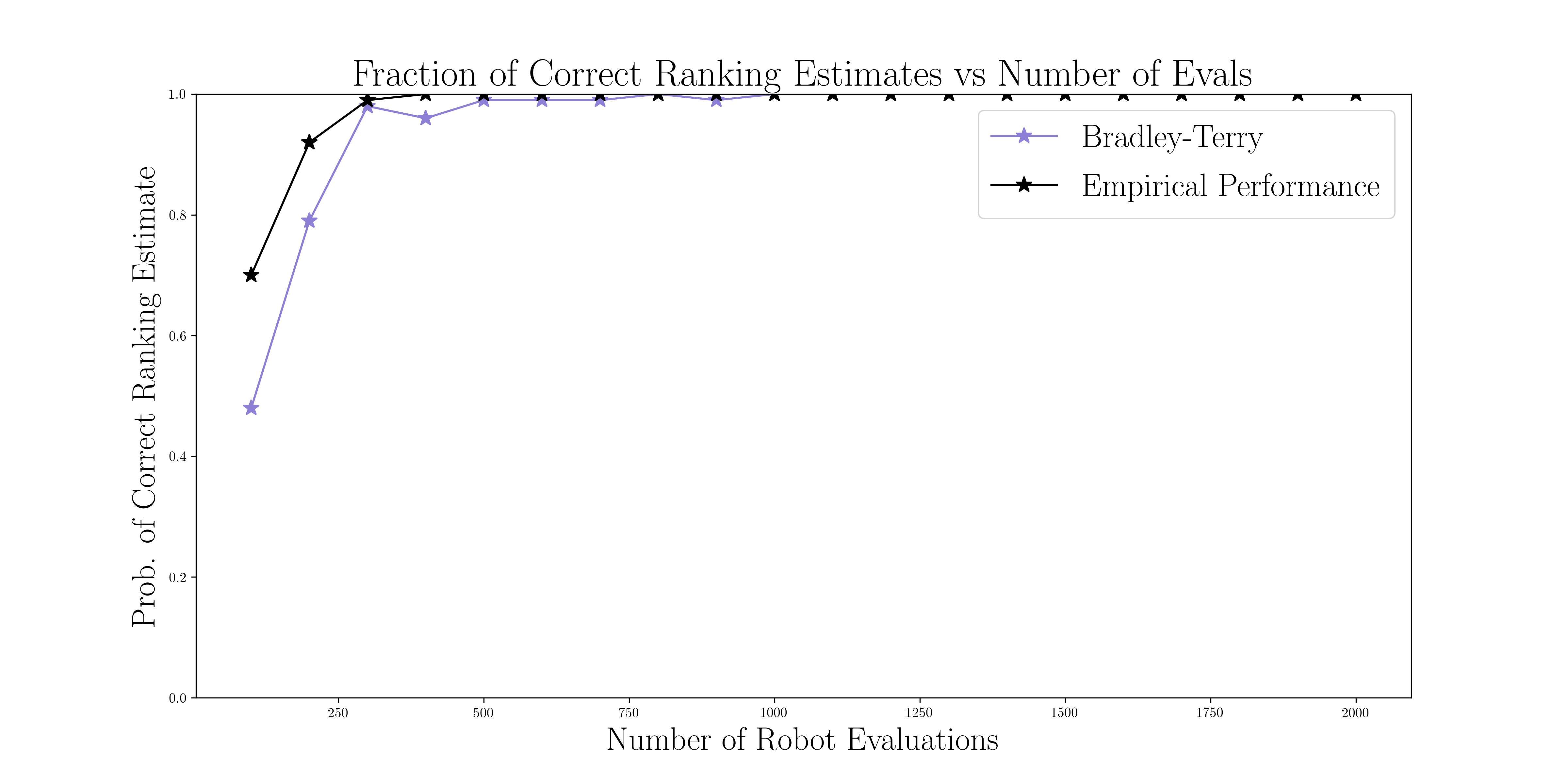}
    \caption{\textbf{Empirical Performance is As Good as Preferences (Feasible Inhomogeneous).}}
    \label{fig:invalid_approaches-FRF_I}
\end{figure}

\begin{figure}
    \centering
    \includegraphics[width=0.9\linewidth]{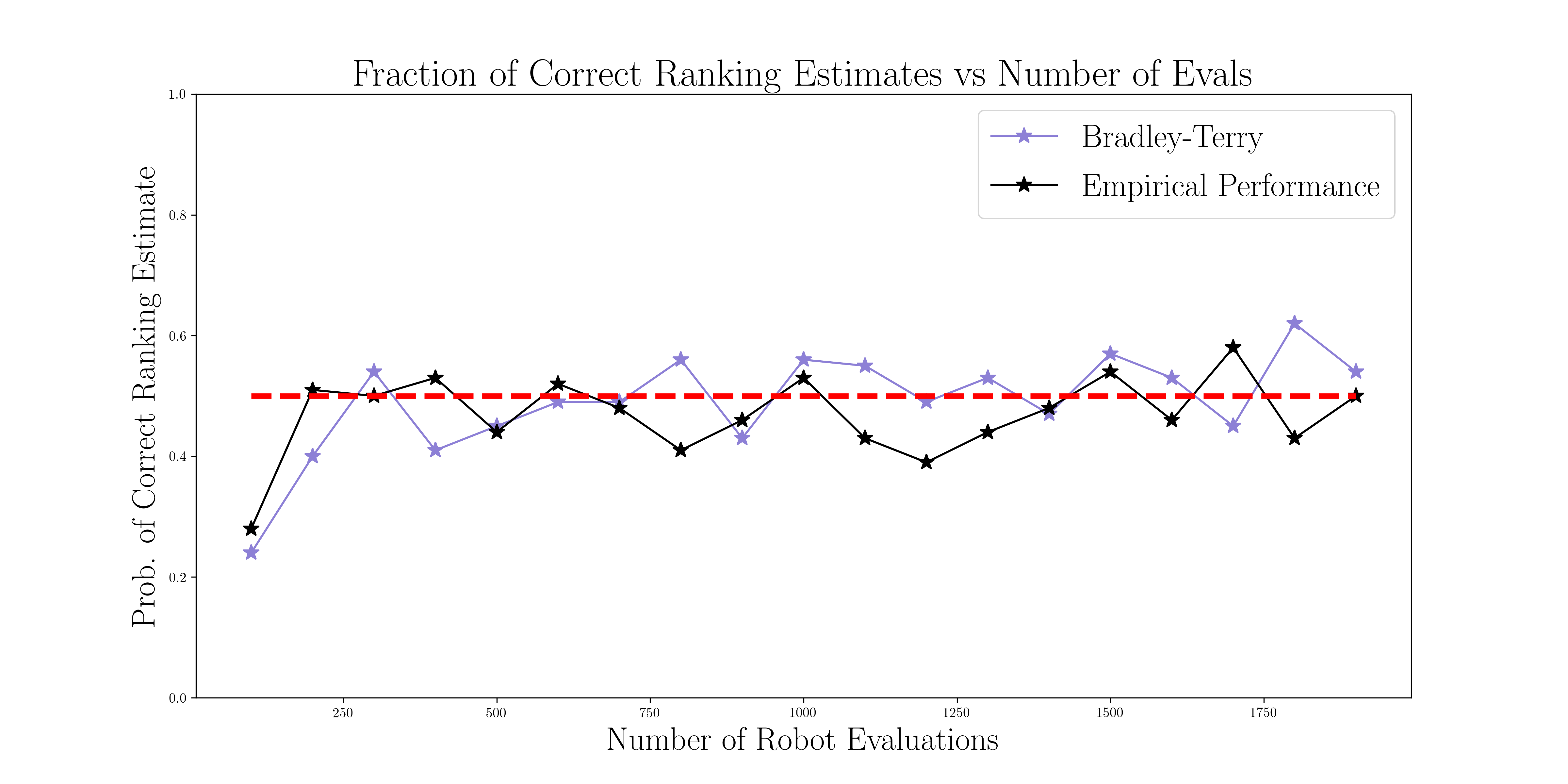}
    \caption{\textbf{Empirical Performance is As Good as Preferences (Infeasible).}}
    \label{fig:invalid_approaches_FRI}
\end{figure}

These trends are borne out in the feasible settings as well. Beginning with \textbf{HE-F} in \Cref{fig:power_heterogeneous_synthetic}, we observe that for gaps of varying difficulty, the presence of necessarily smaller and harder-to-distinguish gaps (e.g., between $\mu_2 = 0.35$ and $\mu_3 = 0.45$) only amplifies the strengths of \algname over BT$_\alpha$. We observe that \algname achieves full empirical ranking power at approximately $1,000$ evaluations, whereas BT$_\alpha$, though it confidently rejects most of the hypotheses, struggles to progress to full rankings, only reaching 60\% power after $2,000$ evaluations. Continuing to \textbf{HO-F} in \Cref{fig:power_homogeneous_synthetic}, we see that even in the best case for BT$_\alpha$ (with homogeneous gaps), the difference in performance remains significant. Here, \algname achieves full empirical power in approximately $750$ evaluations, whereas BT$_\alpha$ requires just over $2,000$ to reach that level. 

\begin{figure}
    \centering
    \includegraphics[width=0.9\linewidth]{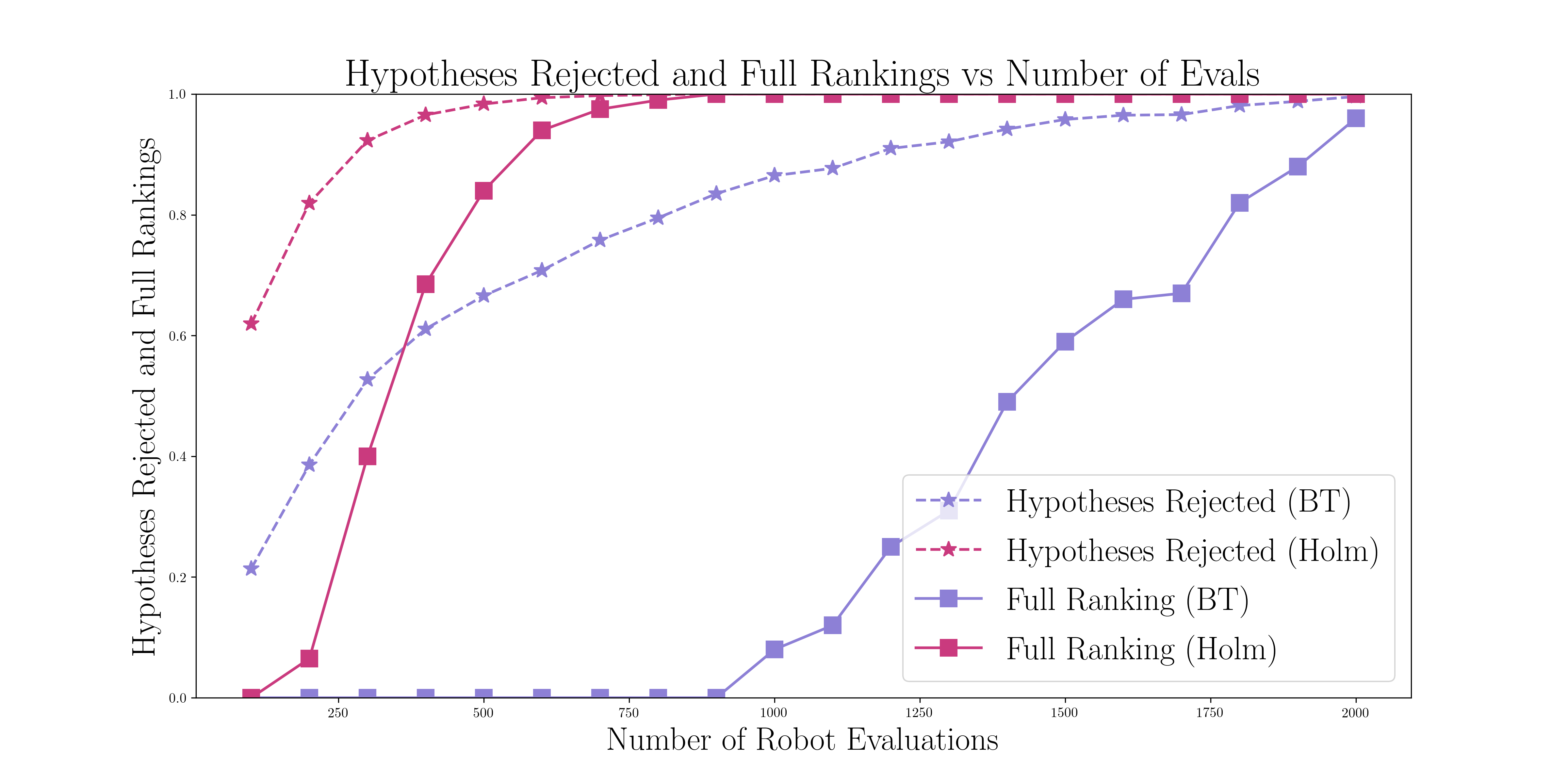}
    \caption{\textbf{Testing Statistical Power with Synthetic Data.} Setting \textbf{HO-F}: we demonstrate the improved efficiency of~$\algname_{0}$~over BT$_\alpha$ on synthetic data. Here, there are five policies corresponding to 10 pairwise comparisons. The policies have equispaced means $\{0.1, 0.3, 0.5, 0.7, 0.9\}$, resulting in an approximately homogeneous \emph{feasible} full ranking. We demonstrate here the time to $k$-hypothesis partial rankings (dashed lines) for each approach, as well as the empirical time to \emph{full ranking} (solid lines). We observe that~$\algname_{0}$~is able to better utilize the risk budget $\alpha = 0.05$ and active sampling in order to more quickly arrive at partial and complete rankings across the spectrum of sample complexity, illustrated by the blue lines being above their red counterparts.}
    \label{fig:power_homogeneous_synthetic}
\end{figure}


\subsection{Counterexample and Bradley-Terry Variants}
We repeat the empirical validation on the counterexample presented in the preceding section, but using the specific EM variant of Bradley-Terry procedures presented in \cite{atreya2025roboarena}. Using the bootstrap modification on dataset sizes of 1000 comparison draws, we observe failed FWER control on five of six simulation reruns (with indeterminate ranking on the remaining run) at nominal level $\alpha = 0.05$; this is significant at $\delta = 0.001$ against the hypothesis that boostrapped EM Bradley-Terry is indeed maintaining FWER control at level $\alpha = 0.05$, via the use of binomial confidence interval-induced tests \cite{vincent_how_2024}. As such, the EM Bradley-Terry procedure also fails to correctly rank the policies by mean performance, despite consistent ranking selection at all times. This illustrates clearly the flattening effect of comparison-based approaches -- the magnitude of performance differences can be lost, which can skew results significantly when those magnitudes are skewed or asymmetric. 

}

\end{document}